\documentclass[letterpaper,onecolumn,draftcls]{IEEEtran}

\usepackage[T1]{fontenc}
\usepackage[utf8]{inputenc}
\usepackage{amsmath,amssymb,amsfonts,amsthm,bm}
\usepackage{algorithm}
\usepackage{algpseudocode}
\usepackage{booktabs}
\usepackage{caption}
\usepackage{cite}
\usepackage{ushort}  
\usepackage{enumitem}
\usepackage{tikz}
\usepackage{import}
\usepackage{graphicx}
\usepackage{microtype}
\usepackage{mathrsfs}
\usepackage{nicefrac}
\usepackage{subcaption}
\usepackage{url}
\usepackage{xcolor}
\usepackage{hyperref}
\usepackage{yhmath}

\usetikzlibrary{arrows.meta, positioning, shapes.geometric}

\theoremstyle{plain}
\newtheorem{theorem}{Theorem}
\newtheorem{lemma}{Lemma}

\newtheorem{corollary}{Corollary}

\theoremstyle{definition}
\newtheorem{definition}{Definition}
\newtheorem{remark}{Remark}

\newcommand{\dmsmfunc}{\mathcal{F}_{\mathrm{sm}}^{d\to m}}

\newcommand{\smop}[3]{s_{#1,#2}[#3]}

\newcommand{\ldecadv}{\ldec^{\mathrm{adv}}}
\newcommand{\ldecgen}{\ldec^{\mathrm{gen}}}
\newcommand{\objadv}{J_\textrm{adv}}
\newcommand{\gcc}{\texttt{GCC}}
\newcommand{\enc}{u_{\textrm{enc}}}
\newcommand{\dec}{u_{\textrm{dec}}}
\newcommand{\encstar}{u_{\textrm{enc}}^{*}}
\newcommand{\decstar}{u_{\textrm{dec}}^{*}}
\newcommand{\encdagger}{u_{\textrm{enc}}^{\dagger}}
\newcommand{\decclean}{u_{\textrm{dec}}^{\circ}}
\newcommand{\enclamb}{\lambda_{\textrm{e}}}
\newcommand{\declamb}{\lambda_{\textrm{d}}}
\newcommand{\lenc}{\mathcal{L}_{\mathrm{enc}}}
\newcommand{\ldec}{\mathcal{L}_{\mathrm{dec}}}
\newcommand{\Radv}{\mathcal{L}_{q}}
\newcommand{\Aadv}{\mathcal{A}_{q}}
\newcommand{\fhat}{\widehat{f}}
\newcommand{\norm}[1]{\lVert #1\rVert}
\newcommand{\abs}[1]{\lvert #1\rvert}

\newcommand{\Leb}{\operatorname{Leb}}
\newcommand{\op}{\mathrm{op}}

\newcommand\delmax[1]{\bar{\delta}_{#1}}
\newcommand\delmin[1]{\ushort{\delta}_{#1}}

\newcommand{\lp}[1]{{L}^{#1}\left(\Omega; \mathbb{R}\right)}
\newcommand{\lpm}[1]{{L}^{#1}\left(\Omega; \mathbb{R}^M\right)}
\newcommand{\lpd}[2]{{L}^{#1}\left(\Omega; \mathbb{R}^{#2}\right)}

\newcommand{\soblocm}[2]{\mathbb{W}_{\textrm{loc}}^{#1,#2}\left(\Omega; \mathbb{R}^M\right)}
\newcommand{\soblocmd}[3]{\mathbb{W}_{\textrm{loc}}^{#1,#2}\left(\Omega; \mathbb{R}^{#3}\right)}

\newcommand{\sobm}[2]{\mathbb{W}^{#1,#2}\left(\Omega; \mathbb{R}^M\right)}
\newcommand{\sobmd}[3]{\mathbb{W}^{#1,#2}\left(\Omega; \mathbb{R}^{#3}\right)}

\newcommand{\sobmeq}[2]{\widetilde{\mathbb{W}}^{#1,#2}\left(\Omega; \mathbb{R}^M\right)}

\newcommand{\hil}[1]{\mathcal{H}^{#1}\left(\Omega;\mathbb{R}\right)}
\newcommand{\hilz}[1]{\mathcal{H}_0^{#1}\left(\Omega;\mathbb{R}\right)}
\newcommand{\hilm}[1]{\mathcal{H}^{#1}\left(\Omega;\mathbb{R}^M\right)}
\newcommand{\hild}[2]{\mathcal{H}^{#1}\left(\Omega;\mathbb{R}^{#2}\right)}
\newcommand{\hilmz}[1]{\mathcal{H}_0^{#1}\left(\Omega;\mathbb{R}^M\right)}

\newcommand{\lec}[2]{\stackrel{\text{#1}}{#2}}

\begin{document}

\title{Learning-Theoretic Foundation for General Coded Computing: The Adversarial Setting}

\author{Parsa Moradi, Hanzaleh Akbarinodehi, and Mohammad Ali Maddah-Ali%
\thanks{P. Moradi, H. Akbarinodehi, and M. A. Maddah-Ali are with the University of Minnesota, Twin Cities, Minneapolis, MN, USA.}%
\thanks{The work has been partially supported by the National Science Foundation under Grant CCF-2348638.}%
\thanks{A preliminary version of this work was presented in ISIT 2025 conference.}
}

\maketitle

\begin{abstract}
Adversarial workers pose a severe challenge in distributed computing systems by returning corrupted outputs that can substantially distort the recovered computations. Coded computing addresses this threat by introducing redundancy across worker assignments. Existing exact and approximate coded-computing methods typically handle adversarial outputs by explicitly locating and correcting or removing erroneous evaluations, which adds significant computational overhead at the master node. Moreover, existing analyses do not characterize how the recovery error scales with the number of workers and the adversarial budget.

In this paper, we develop an adversarially robust extension of general coded computing (\gcc{}), a learning-theoretic approximate coded-computing framework that designs the encoder and decoder by minimizing an upper bound on the end-to-end recovery error. Rather than identifying corrupted workers, the proposed method uses smoothness regularization to limit the influence of their outputs. The resulting encoder and decoder can be implemented using smoothing splines with linear master-side complexity. We derive a non-asymptotic bound on the worst-case end-to-end approximation error by decomposing it into three components: (1) the encoder approximation error, (2) the decoder reconstruction error under fully honest results, and (3) the adversarial perturbation error. For maximum of $q\geq 1$ adversarial workers among $N$ workers results in a worst-case error of order $\mathcal{O}((q/N)^{6/5})$ with appropriate selection of smoothness parameter, demonstrating that the approximation error vanishes asymptotically whenever the number of adversaries grows sublinearly with the network size. Experiments on scalar computation and neural-network inference are consistent with the predicted convergence behavior.
\end{abstract}


\IEEEpeerreviewmaketitle

\section{Introduction}\label{sec:introduction}

The growing scale of data-intensive applications and computationally
demanding machine-learning (ML) models has made distributed
computation central to modern computing systems~\cite{verbraeken2020survey}. Such systems
commonly operate under a \emph{master--worker architecture}, in which
a central \emph{master node} assigns computational tasks to a
collection of \emph{worker nodes}. Each worker node performs its assigned
computation and returns the result to the master node, which aggregates the
worker outputs to recover the desired function evaluations. By
executing these tasks in parallel, the system can substantially reduce
the computation time incurred at the master node.

This benefit, however, comes with a reliability
challenge: some workers may fail to return their results (known as straggler worker), while others
may return incorrect values because of faults or malicious behavior (known as adversarial worker).
These two forms of unreliability are fundamentally different. A
straggler removes an observation, whereas an adversarial worker
introduces a corrupted observation that may be coordinated with other
corruptions and deliberately chosen to distort the final estimate. The
master node must therefore reconstruct the desired function values without
knowing which worker outputs, if any, are corrupted.

Coded computing addresses worker unreliability by introducing
redundancy across the assigned computations. The master node encodes the
original inputs into some \emph{coded inputs} and send each coded inputs to one worker nodes. Next, each worker node apply the target function on its assigned coded input and return the results to the master node. Finally, the master node reconstructs the desired computations from the worker
outputs~\cite{Abraham1984, lee2017speeding, yu2019lagrange, yu2020straggler,short,  opt-recovery, jia2019capacity} even though some of them are missing or incorrect. 

For polynomial computations, Lagrange Coded Computing (LCC) achieves exact recovery while simultaneously providing resilience against stragglers and robustness against adversarial workers whenever the number of worker nodes exceeds a prescribed \emph{recovery threshold}~\cite{yu2019lagrange}. This threshold is determined by the degree of the target polynomial computation as well as the maximum numbers of straggling and adversarial nodes. However, LCC is originally formulated over finite fields and does not scale well with the input batch size due to numerical overflow and the severe ill-conditioning of high-degree polynomial interpolation~\cite{jahani2022berrut,gautschi1987lower}. To address this limitation, Analog Lagrange Coded Computing (ALCC) extends LCC to infinite fields (the analog domain), preserving straggler resilience and dataset privacy~\cite{soleymani_analog}. 

Approximate coded computing extends coded computation to a broader
class of functions by allowing approximate recovery. Berrut
approximated coded computing (\texttt{BACC}) uses barycentric rational
interpolation to approximate general nonlinear computations from the
non-straggler worker outputs~\cite{jahani2022berrut}. More recently, the
learning-theoretic general coded computing (\gcc{}) framework designs
the encoder and decoder by deriving and minimizing an upper bound on
the average end-to-end approximation error~\cite{moradi2026learning}.

Although \gcc{} provides an efficient scheme that outperforms \texttt{BACC} with a faster convergence rate and provably mitigates the effect of stragglers, its robustness against adversarial workers remains unexplored. Conversely, adversarial robustness has been studied within the Berrut framework. ApproxIFER combines rational interpolation with an algebraic error-location algorithm inspired by Berlekamp--Welch decoding to provide Byzantine-robust distributed inference~\cite{soleymani2022approxifer}. 
However, these approaches address adversarial corruption through explicit error-location and correction procedures \emph{prior} to reconstructing the desired outputs, which imposes substantial computational overhead at the master node. Furthermore, none of these works investigates how the approximation error behaves as the number of workers increases for a given adversarial budget, nor do they establish fundamental limits on the adversarial budget required for the approximation error to vanish asymptotically.

To address these gaps, we investigate the adversarial robustness of the \gcc{} framework from a complementary perspective: rather than explicitly locating and discarding adversarial workers, we mitigate their influence by inducing \emph{smoothness} in the encoding and decoding procedures. Building upon the \gcc{} architecture, the proposed method utilizes smooth functions to distribute computations and reconstruct the desired outputs. The decoder aggregates all returned outputs and bounds the impact of corrupted values via smoothness regularization. While the regularization parameters can be calibrated to a known adversarial budget, the framework operates without prior knowledge of either the adversaries' identities or their specific corruptions.

Crucially, we establish that with an appropriately tuned smoothness regularization, the end-to-end approximation error converges to zero as the total number of workers increases, provided the number of adversarial workers grows sub-linearly with the network size.

Our main contributions are summarized as follows:

\begin{enumerate}
\item \textbf{Adversarial \gcc{} Framework:}
We propose an adversarially robust \gcc{} scheme for general high-dimensional computations that operates reliably in the presence of adversarial workers. We derive an upper bound on the end-to-end approximation error and decompose it into three components: (i) the encoder approximation error, (ii) the decoder reconstruction error in a fully honest setting (referred to as the \emph{clean approximation error}), and (iii) the perturbation caused by adversarial workers. This decomposition explicitly reveals how the smoothness of the encoder and decoder governs the robustness of the overall scheme.

\item \textbf{Finite-Sample Guarantees and Convergence Rates:} 
We derive a non-asymptotic upper bound on the worst-case mean-squared approximation error. This bound captures the fundamental trade-off introduced by decoder regularization: stronger smoothing reduces the influence of adversarial outputs but increases the clean approximation error. Balancing these two effects reveals that for a maximum of $q\geq1$ adversarial workers out of $N$ total workers, the decoder smoothing parameter should scale as $(q/N)^{8/5}$. With this tuning, the proposed scheme achieves an end-to-end error rate of $\mathcal{O}((q/N)^{6/5})$, demonstrating that the worst-case approximation error converges to zero whenever the number of adversarial workers grows sub-linearly with the total number of workers.

\item \textbf{Efficient Implementation:} 
We show that both the encoder and decoder can be implemented using vector-valued cubic smoothing splines. Because their coefficients can be computed by solving banded linear systems, the master-side encoding and decoding costs scale linearly with both the input batch size and the total number of workers. Notably, this computational efficiency is achieved without requiring a separate adversarial detection or error-location subroutine.

\item \textbf{Empirical Evaluation:} 
We evaluate the proposed scheme on both scalar nonlinear computations and high-dimensional neural network inference. The experiments illustrate the predicted convergence behavior under adversarial attacks and confirm that the observed error decay aligns with our theoretical guarantees.
\end{enumerate}

The remainder of the paper is organized as follows. Section~\ref{sec:problem} introduces the problem formulation. Section~\ref{sec:prelim} reviews the necessary mathematical preliminaries regarding Sobolev spaces, reproducing kernels, and smoothing splines. Section~\ref{sec:nonasymptotic} develops the proposed encoding and decoding procedures, establishes their error and convergence guarantees, analyzes the computational complexity. The proofs of the main theorems and corollaries are collected in Section~\ref{sec:proofs}, and Section~\ref{sec:experiments} presents the experiments result.

\subsection{Notation}

For $n\in\mathbb{N}$, let
$
[n]:=\{1,\ldots,n\}.
$
The cardinality of a set $\mathcal{S}$ is denoted by $|\mathcal{S}|$. Scalars are denoted by lowercase letters, such as $v$, vectors by bold lowercase letters, such as $\mathbf{v}$, and matrices by bold uppercase letters, such as $\mathbf{V}$. The Euclidean norm is denoted by $\norm{\cdot}_2$. For a vector $\mathbf{v}\in\mathbb{R}^m$, its coordinatewise maximum norm is defined as
$
\norm{\mathbf{v}}_{\infty}
:=
\max_{j\in[m]}\abs{v_j}.
$
For a matrix $\mathbf{A}$, its induced operator norm is defined as
$
\norm{\mathbf{A}}_{\op}
:=
\sup_{\norm{\mathbf{v}}_2=1}
\norm{\mathbf{A}\mathbf{v}}_2
$. We write $a \asymp b$ to indicate that there exist positive constants $c_1, c_2 > 0$ such that $c_1 a \leq b \leq c_2 a$.

For a twice differentiable scalar-valued function
$g:\mathbb{R}^d\to\mathbb{R}$,
$\nabla g(\mathbf{x})$ and $D^2g(\mathbf{x})$ denote its gradient and Hessian, respectively. For a twice differentiable vector-valued function
$g=[g_1,\ldots,g_m]^T$,
$\nabla g_j(\mathbf{x})$ and $D^2g_j(\mathbf{x})$ denote the gradient and Hessian of its $j$th component, respectively.

For a measurable function $u:\Omega\to\mathbb{R}^M$, $u'$ and $u''$ denote its first and second weak derivatives, taken componentwise with respect to its scalar argument. The formal definition of a weak derivative is provided in Definition~\ref{def:weak_deriv} in Appendix~\ref{app:functional_details}. The Lebesgue measure of a measurable set $\mathcal{E}\subseteq\mathbb{R}$ is denoted by $\Leb(\mathcal{E})$. For $1\leq p\leq\infty$, $\norm{u}_{\lpm{p}}$ denotes its $L^p$-norm over $\Omega$, as formally defined in Definition~\ref{def:lp_norm} in Appendix~\ref{app:functional_details}.

\section{Problem Formulation}\label{sec:problem}

We consider a master-worker distributed computing system consisting of
one master node and $N$ worker nodes. The master node aims to evaluate
a prescribed computing function
\(
    f:\mathbb{R}^d\to\mathbb{R}^m
\)
on a collection of $K$ input data points
$\{\mathbf{x}_k\}_{k=1}^K$, using the worker nodes, where
$N$, $K$, $d$, and $m$ are positive integers. 

The set of worker nodes is partitioned into two disjoint subsets:
the set of honest workers, denoted by $\mathcal{D}$, and the set of
adversarial workers, denoted by $\mathcal{B}$, where
\(
    \mathcal{D}\cup\mathcal{B}
    =
    [N],
    \mathcal{D}\cap\mathcal{B}
    =
    \varnothing.
\)
Each honest worker evaluates the prescribed function at its assigned
input and returns the resulting output without modification. In
contrast, an adversarial worker may return an arbitrary vector in the
bounded adversarial-output region
\begin{align}
    \mathcal{Y}_{\textrm{adv}}
    :=
    [-T_a,T_a]^m,
\end{align}
for some $M>0$, regardless of the computation assigned to that worker.
We assume that
\begin{align}\label{eq:n_adv_bound}
    |\mathcal{B}|
    \leq
    q,
\end{align}
where $q\in\{0,\ldots,N\}$ is the maximum number of adversarial
workers.

All adversarial workers are controlled by a single entity, referred
to as the \emph{adversary}, and may coordinate their returned
outputs. An attack is specified by the set of corrupted workers and
their returned vectors. Accordingly, define
\begin{align}\label{eq:attack_class}
    \Aadv
    :=
    \bigl\{
        A
        =
        \bigl(
            \mathcal{B},
            \{\mathbf{a}_j\}_{j\in\mathcal{B}}
        \bigr)
        :
        \mathcal{B}\subseteq[N],\,
        |\mathcal{B}|\leq q,\,
        \mathbf{a}_j\in\mathcal{Y}_{\textrm{adv}}
    \bigr\}.
\end{align}
The adversary has full knowledge of the input batch
$\{\mathbf{x}_k\}_{k=1}^K$, the computing function $f$, and the complete encoding and decoding
procedures. Thus, the adversary may choose the corrupted worker set and
all corrupted outputs jointly.

\subsection{General Coded Computing in the Adversarial Setting}
\label{sec:adv_gcc_model}

Similar to \gcc{} in the straggler setting~\cite{moradi2026learning}, the
proposed framework consists of three stages all depicted in Figure~\ref{fig:framework}: (i) encoding at the master
node, (ii) computation at the worker nodes, and (iii) decoding at the master
node.
\begin{enumerate}[label={\rm(\arabic*)},leftmargin=1.8em]

\item \emph{Encoding:}
The master node selects $K$ fixed scalar points
\(
    \alpha_1
    <
    \cdots
    <
    \alpha_K
\)
with
\(\alpha_k\in\Omega,\) for $k\in[K]$
referred to as the \emph{encoder design points}, and constructs an
encoder function
\(
    \enc:\Omega\to\mathbb{R}^d
\)
such that
\begin{align}\label{eq:enc_approx}
    \enc(\alpha_k)
    \approx
    \mathbf{x}_k,
    \qquad
    k\in[K].
\end{align}
The master node also selects $N$ fixed scalar points
\(
    \beta_1
    <
    \cdots
    <
    \beta_N,
\)
with
\(
\beta_j\in\Omega,
\) for $j\in[N]$
referred to as the \emph{decoder design points}. It forms the coded
inputs by evaluating the encoder at these points:
\begin{align}\label{eq:coded_inputs}
    \widetilde{\mathbf{x}}_j
    :=
    \enc(\beta_j),
    \qquad
    j\in[N].
\end{align}
The master node then assigns
$\widetilde{\mathbf{x}}_j$ to worker $j$.

\item \emph{Computation:}
Every honest worker $j\in\mathcal{D}$ evaluates the prescribed
computing function at its assigned coded input and returns
\(
    f(\widetilde{\mathbf{x}}_j)
    =
    f(\enc(\beta_j)).
\)
In contrast, an adversarial worker $j\in\mathcal{B}$ may replace this
value by an arbitrary vector
$\mathbf{a}_j\in\mathcal{Y}_{\textrm{adv}}$. Therefore, under an attack
$A\in\Aadv$, the output received by the master node from worker $j$ is
\begin{align}\label{eq:received_output}
    \overline{\mathbf{y}}_j^{A}
    =
    \begin{cases}
        f(\enc(\beta_j)),
        &
        j\in\mathcal{D},
        \\[1mm]
        \mathbf{a}_j\in\mathcal{Y}_{\textrm{adv}},
        &
        j\in\mathcal{B}.
    \end{cases}
\end{align}

\item \emph{Decoding:}
Using the returned outputs
$\{\overline{\mathbf{y}}_j^{A}\}_{j=1}^N$ and the corresponding decoder
design points $\{\beta_j\}_{j=1}^N$, the master node constructs a
decoder function
\(
    \dec:\Omega\to\mathbb{R}^m.
\)
The decoder function is designed to suppress the effect of the adversarial
outputs and accurately approximate the composed function
$f\circ\enc$ in the presence of up to $q$ corrupted worker outputs. In particular, the objective is to have a decoder function
$\dec(\cdot)$ such that
\begin{align}\label{eq:dec_approx}
    \dec(z)
    \approx
    f(\enc(z)),
    \qquad
    z\in\Omega.
\end{align}

The master node then estimates the desired computation at each
original input point by evaluating the decoder at the corresponding
encoder design point:
\begin{align}\label{eq:gcc_estimate}
    \fhat_q(\mathbf{x}_k)
    &:=
    \dec(\alpha_k)
    \lec{(a)}{\approx}
    f(\enc(\alpha_k))
    \lec{(b)}{\approx}
    f(\mathbf{x}_k),
    \qquad
    k\in[K],
\end{align}
where approximation (a) follows from the decoder approximation in
\eqref{eq:dec_approx}, while approximation (b) follows from the
encoder approximation in \eqref{eq:enc_approx} and the regularity of
$f$. These two sources of error will be formally quantified in the Section~\ref{sec:code_design}.
\end{enumerate}

For a fixed computing function $f$, prescribed encoder design points
$\boldsymbol{\alpha}:=\{\alpha_k\}_{k=1}^K$, prescribed decoder
design points
$\boldsymbol{\beta}:=\{\beta_j\}_{j=1}^N$, and an adversarial budget
$q$, we measure the performance of the \gcc{} scheme by the
worst-case mean squared error (MSE)
\begin{align}\label{eq:main_obj}
    \Radv(
        \enc,\dec;
        f,\boldsymbol{\alpha},\boldsymbol{\beta}
    )
    :=
    \sup_{A\in\Aadv}
    \frac{1}{K}
    \sum_{k=1}^K
    \norm{
        \fhat_q(\mathbf{x}_k)
        -
        f(\mathbf{x}_k)
    }_2^2.
\end{align}
Throughout our analysis,
$\boldsymbol{\alpha}$, and $\boldsymbol{\beta}$ are fixed. We
therefore suppress their dependence and write
\begin{align}
    \Radv(\enc,\dec;f)
    :=
    \Radv(
        \enc,\dec;
        f,\boldsymbol{\alpha},\boldsymbol{\beta}
    ).
\end{align}
The supremum in \eqref{eq:main_obj} is taken over all admissible
attacks that corrupt the outputs of at most $q$ workers. The encoder
and decoder may be designed using the known budget $q$, but they do
not know the realized set of adversarial workers or the corrupted
values returned by them.

Accordingly, the objective of \gcc{} in adversarial environment is to design the
encoder and decoder functions to minimize the worst-case error in
\eqref{eq:main_obj}:
\begin{align}\label{eq:obj_adv_gcc}
    \inf_{\substack{
        \enc,
        \dec
    }}
    \Radv(\enc,\dec;f).
\end{align}


\begin{figure}[!t]
\centering
\resizebox{\columnwidth}{!}{%
\begin{tikzpicture}[
    >=Stealth,
    worker/.style={
        rectangle,
        draw,
        thick,
        fill=gray!5,
        minimum width=2.4cm,
        minimum height=0.75cm,
        align=center,
        font=\small
    },
    adversarial/.style={
        worker,
        draw=red!70!black,
    },
    link/.style={
        ->,
        thick
    },
    corrupted/.style={
        link,
        dashed,
        red!70!black
    }
]

\draw[thick, fill=gray!5]
    (0,0)
    -- (2.5,1.5)
    -- (2.5,-4.6)
    -- (0,-3.1)
    -- cycle;

\node[font=\bfseries] at (1.25,-0.9)
    {Encoder};

\node[align=center, font=\small] at (1.25,-2.1)
    {$\widetilde{\mathbf{x}}_j=\enc(\beta_j)$};

\draw[thick, fill=gray!5]
    (8.5,1.5)
    -- (11,0)
    -- (11,-3.1)
    -- (8.5,-4.6)
    -- cycle;

\node[font=\bfseries] at (9.75,-0.9)
    {Decoder};

\node[align=center, font=\small] at (9.75,-2.1)
    {$\fhat(\mathbf{x}_k)=\dec(\alpha_k)$};

\node[font=\bfseries] at (5.5,1.95)
    {Worker nodes};

\foreach \name/\idx/\yy in {
    w1/1/1.2,
    w2/2/0.1,
    wn/n/-2.1,
    wN/N/-4.3
}{
    \node[worker] (\name) at (5.5,\yy)
        {$f(\cdot)$};

    \draw[link]
        (2.5,\yy)
        -- node[above, font=\small] (input-\idx)
           {$\widetilde{\mathbf{x}}_{\idx}$}
        (\name.west);

    \draw[link]
        (\name.east)
        -- node[above, font=\small]
           {$f(\widetilde{\mathbf{x}}_{\idx})$}
        (8.5,\yy);
}

\foreach \name/\idx/\yy in {
    bi/i/-1.0,
    bj/j/-3.2
}{
    \node[adversarial] (\name) at (5.5,\yy)
        {\includegraphics[height=0.50cm]{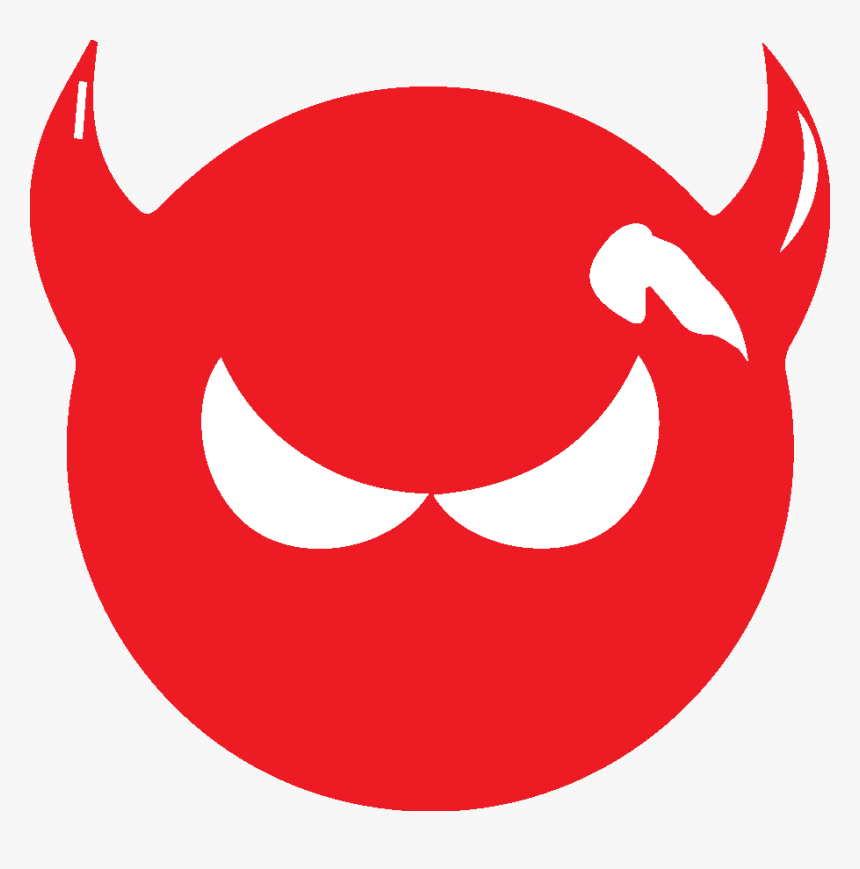}};

    \draw[link]
        (2.5,\yy)
        -- node[above, font=\small] (input-\idx)
           {$\widetilde{\mathbf{x}}_{\idx}$}
        (\name.west);

    \draw[corrupted]
        (\name.east)
        -- node[above, font=\small](output-\idx)
           {$\mathbf{a}_{\idx}$}
        (8.5,\yy);
}


\foreach \idx/\yy in {
    1/-0.5,
    2/-1.2,
    K/-2.6
}{
    \draw[link]
        (-0.8,\yy)
        node[left] {$\mathbf{x}_{\idx}$}
        -- (0,\yy);
}

\node at (-1.05,-1.9) {$\vdots$};

\foreach \idx/\yy in {
    1/-0.5,
    2/-1.2,
    K/-2.6
}{
    \draw[link]
        (11,\yy)
        -- (11.8,\yy)
        node[right] {$\fhat(\mathbf{x}_{\idx})$};
}

\node at (12.35,-1.9) {$\vdots$};

\end{tikzpicture}%
}

\caption{General coded computing in adversarial setting.
Honest workers return $f(\widetilde{\mathbf{x}}_j)$, whereas
adversarial workers may return arbitrary vectors
$\mathbf{a}_j\in\mathcal{Y}_{\mathrm{adv}}$.
The adversary icons and dashed red arrows illustrate corrupted
outputs. The adversarial workers' identities are unknown to the
master node.}
\label{fig:framework}
\end{figure}

\section{Preliminaries}\label{sec:prelim}

In this section, we introduce the notation and mathematical foundations required for the design and analysis of the proposed encoder and decoder. 

Let $\Omega = [a, b] \subset \mathbb{R}$ denote a bounded closed interval. For $1 \leq p \leq \infty$, let $\lpm{p}$ denote the space of measurable functions on $\Omega$ with finite $L^p$-norm.

\subsection{Sobolev Spaces}

For an integer $r \geq 1$ and $1 \leq p \leq \infty$, the Sobolev space $\sobm{r}{p}$ consists of functions $g \in \lpm{p}$ whose weak derivatives up to order $r$, denoted $g^{(i)}$ for $i = 1, \ldots, r$, also belong to $\lpm{p}$. Setting $g^{(0)} := g$, the norm on $\sobm{r}{p}$ is defined by
\begin{align}\label{eq:sob_norm1}
    \norm{g}_{\sobm{r}{p}}
    :=
    \bigg(
        \sum_{i=0}^r
        \norm{g^{(i)}}_{\lpm{p}}^p
    \bigg)^{1/p},
    \qquad 1 \leq p < \infty,
\end{align}
and for $p = \infty$,
\begin{align}\label{eq:sob_norm2}
    \norm{g}_{\sobm{r}{\infty}}
    :=
    \max_{0 \leq i \leq r}
    \norm{g^{(i)}}_{\lpm{\infty}}.
\end{align}
When $p = 2$, $\sobm{r}{2}$ forms a Reproducing Kernel Hilbert Space (RKHS)~\cite{berlinet2011reproducing, wahba1990spline}, which we denote by
\begin{align}
    \hilm{r} := \sobm{r}{2}.
\end{align}
Throughout the paper, $\hilm{r}$ denote the $r$-th order Sobolev space over interval $\Omega$.

\subsubsection{Equivalent Norms}\label{sec:equiv_sob_norm}

A Sobolev space can be equipped with several equivalent
norms~\cite{adams2003sobolev,berlinet2011reproducing,wahba1990spline}.
Two norms $\norm{\cdot}_{\mathbb{W}_1}$ and
$\norm{\cdot}_{\mathbb{W}_2}$ on the same space are equivalent if
there exist constants $\theta_1,\theta_2>0$ such that
\(
    \theta_1\norm{g}_{\mathbb{W}_1}
    \leq
    \norm{g}_{\mathbb{W}_2}
    \leq
    \theta_2\norm{g}_{\mathbb{W}_1}
\)
for every function $g$ in that space. 

The equivalent norm used in this paper is the anchored Sobolev norm
introduced in~\cite{kimeldorf1971some}. Let $\Omega=[a,b]$ and
$1\leq p<\infty$. For every $g\in\sobm{r}{p}$, define
\begin{align}\label{eq:sob_norm_eq}
    \norm{g}_{\sobmeq{r}{p}}
    :=
    \left(
        \sum_{i=0}^{r-1}
        \norm{g^{(i)}(a)}_2^p
        +
        \norm{g^{(r)}}_{\lpm{p}}^p
    \right)^{1/p}.
\end{align}
This norm is equivalent to the standard Sobolev norm in
\eqref{eq:sob_norm1}~\cite{moradicoded}. In particular, since
$\hilm{r}=\sobm{r}{2}$, throughout the remainder of the paper we use
the convention
\begin{align}\label{eq:hilbert_equivalent_norm}
    \norm{g}_{\hilm{r}}
    :=
    \norm{g}_{\sobmeq{r}{2}},
    \qquad
    g\in\hilm{r}.
\end{align}

\begin{remark}
    We define $\Omega = [a, b]$ as a closed interval; however, the underlying $L^p$ and Sobolev norms are identical whether defined over $[a, b]$ or the open interval $(a, b)$, as the boundary set $\{a, b\}$ has Lebesgue measure zero. Thus, we have:
    \begin{align}
        \mathcal{H}^r([a,b], \mathbb{R}^M) &= \mathcal{H}^r((a,b), \mathbb{R}^M) \nonumber \\ 
        L^r([a,b], \mathbb{R}^M) &= L^r((a,b), \mathbb{R}^M)
    \end{align}
\end{remark}

\subsection{Representer Theorem}

The representer theorem states that the solution to a broad class of regularized empirical risk minimization problems over an RKHS admits a finite-dimensional representation, expressed as a linear combination of kernel evaluations at the sampling points.

\begin{theorem}[Representer Theorem~\cite{scholkopf2001generalized, micchelli2005learning}]
\label{thm:representer}
Let $\mathcal{H}$ be a vector-valued RKHS of functions from $\Omega$ to $\mathbb{R}^d$, associated with a matrix-valued reproducing kernel $\Phi: \Omega \times \Omega \to \mathbb{R}^{d \times d}$. Let $\{t_i\}_{i=1}^n \subset \Omega$, $\lambda > 0$, and let $\Theta: [0, \infty) \to \mathbb{R}$ be a strictly increasing function. Consider the regularized optimization problem
\begin{align}\label{eq:representer_objective}
    \min_{g \in \mathcal{H}} \;
    \Psi\big(g(t_1), \ldots, g(t_n)\big)
    +
    \lambda \Theta\big(\norm{g}_{\mathcal{H}}\big),
\end{align}
where $\Psi: (\mathbb{R}^d)^n \to \mathbb{R} \cup \{+\infty\}$ is an arbitrary loss function. If a minimizer of~\eqref{eq:representer_objective} exists, it admits the representation
\begin{align}\label{eq:representer}
    g^*(t)
    =
    \sum_{i=1}^n
    \Phi(t, t_i) \mathbf{w}_i,
\end{align}
for some coefficient vectors $\mathbf{w}_1, \ldots, \mathbf{w}_n \in \mathbb{R}^d$.
\end{theorem}

When $\mathcal{H} = \hild{r}{d}$, the reproducing kernel decouples as $\Phi(t, s) = \phi(t, s) I_d$, where $\phi: \Omega \times \Omega \to \mathbb{R}$ denotes the scalar reproducing kernel of $\hild{r}{}$, and $I_d$ is the $d \times d$ identity matrix~\cite{aubin2000applied,li2024towards}. Under this setting, the representation~\eqref{eq:representer} simplifies to
\begin{align}
    g^*(t) = \sum_{i=1}^n \phi(t, t_i) \mathbf{w}_i.
\end{align}

\subsection{Second-Order Smoothing Splines}\label{sec:smooth_spline}

Let $a < t_1 < \cdots < t_n < b$, with $n \geq 2$, denote an ordered set of sampling points, and let
$\mathbf{y}
:=
(\mathbf{y}_1,\ldots,\mathbf{y}_n)
\in(\mathbb{R}^d)^n$ be the corresponding vector-valued observations. For a regularization parameter $\lambda > 0$ (known as \emph{smoothing parameter}), consider the penalized least-squares problem over second-order Sobolev space:
\begin{align}\label{eq:smoothing_spline}
    \smop{n}{\lambda}{\mathbf{y}}
    :=
    \underset{g \in \hild{2}{d}}{\operatorname{argmin}}
    \bigg\{
        \frac{1}{n} \sum_{i=1}^n
        \norm{g(t_i) - \mathbf{y}_i}_2^2
        +
        \lambda \norm{g''}_{\lpd{2}{d}}^2
    \bigg\}.
\end{align}
Because both the quadratic empirical loss and the derivative penalty decouple across the $d$ coordinates, the minimization in~\eqref{eq:smoothing_spline} decomposes into $d$ independent scalar problems. For each coordinate, the objective admits a unique minimizer given by a natural cubic smoothing spline (also referred to as a second-order smoothing spline)~\cite{wahba1990spline}. Consequently, \eqref{eq:smoothing_spline} possesses a unique vector-valued solution $\smop{n}{\lambda}{g}$.

Moreover, the solution operator is linear in the observations and admits the equivalent kernel representation~\cite{wahba1975smoothing}
\begin{align}\label{eq:smoothing_kernel}
    \smop{n}{\lambda}{\mathbf{y}}(t)
    =
    \frac{1}{n} \sum_{i=1}^n
    G_{n,\lambda}(t, t_i) \mathbf{y}_i.
\end{align}
The equivalent kernel (or weight function) $G_{n,\lambda}$ depends solely on the design points $\{t_i\}_{i=1}^n$ and the regularization parameter $\lambda$, and is identical across all coordinates.

When the observations are noiseless samples from a function
$h:[a,b]\to\mathbb{R}^d$, let
\(
\mathbf{h}:=[h(t_1),\ldots,h(t_n)]
\)
denote the corresponding observation vector. We then use the shorthand
\(
\smop{n}{\lambda}{h}
:=
\smop{n}{\lambda}{\mathbf{h}}.
\)

\section{\gcc{} in Adversarial Environments}\label{sec:nonasymptotic}

In this section, we formulate the \gcc{} framework under an adversarial setting in which, for a given $q \in \mathbb{N}$, at most $q$ outputs returned by worker nodes are arbitrarily corrupted. We begin by characterizing the spacing of the design points, which plays a central role in our theoretical guarantees.

\begin{definition}[Design Spacing Parameters]\label{def:spacing}
Let $\mathcal{T} := \{t_1, \ldots, t_n\} \subset \Omega$ be an ordered set of design points satisfying $a \leq t_1 < \cdots < t_n \leq b$. The maximum and minimum consecutive spacing of $\mathcal{T}$ are defined, respectively, by
\begin{align}
    \delmax{\mathcal{T}}
    &:=
    \max_{1\leq j < n}
    (t_{j+1} - t_j), \\
    \delmin{\mathcal{T}}
    &:=
    \min_{1\leq j < n}
    (t_{j+1} - t_j).
\end{align}
\end{definition}

For the remainder of this paper, we assume that both the encoder and decoder belong to the second-order vector-valued Sobolev space; that is,
\begin{align}
    \enc \in \hild{2}{d}, \qquad \dec \in \hild{2}{m}.
\end{align}
This regularity assumption imposes only mild smoothness constraints while providing the analytical tractability needed to leverage reproducing kernel techniques. In addition, without loss of generality, we set $\Omega = [0, 1]$.\footnote{Any non-empty bounded interval $[a, b]$ can be mapped to $[0, 1]$ via an affine transformation, which affects our theoretical constants only up to scale factors.}

Next, we define the class of computing functions considered throughout this paper.

\begin{definition}[Admissible Computing Function Class]\label{def:function_class}
We define $\dmsmfunc$ as the class of twice continuously differentiable functions $f = [f_1, \ldots, f_m]^\top : \mathbb{R}^d \to \mathbb{R}^m$ with bounded outputs, gradients, and Hessians:
\begin{align}\label{eq:converse_function_class}
    \dmsmfunc
    :=
    \biggl\{
        f=[f_1,\ldots,f_m]^\top \in C^2(\mathbb{R}^d; \mathbb{R}^m) :
        &\
        \sup_{\mathbf{x}\in\mathbb{R}^d}
        \norm{f(\mathbf{x})}_{\infty}
        < \infty,
        \nonumber\\
        &\
        \sup_{\mathbf{x}\in\mathbb{R}^d}
        \norm{\nabla f_j(\mathbf{x})}_2
        <\infty,
        \quad j\in[m],
        \nonumber\\
        &\
        \sup_{\mathbf{x}\in\mathbb{R}^d}
        \norm{D^2f_j(\mathbf{x})}_{\op}
        <\infty,
        \quad j\in[m]
    \biggr\}.
\end{align}
\end{definition}

The function class $\dmsmfunc$ is highly general and encompasses a wide variety of non-linear functions. For a given computing function $f \in \dmsmfunc$, our objective is to design the encoder and decoder functions to solve the robust risk minimization problem
\begin{align}\label{eq:main_adv_obj}
    \objadv
    :=
    \inf_{\substack{
        \enc \in \hild{2}{d}, \\
        \dec \in \hild{2}{m}
    }}
    \Radv(\enc, \dec;f).
\end{align}

The remainder of this section is organized as follows. In Section~\ref{sec:code_design}, we present the encoder and decoder constructions and establish both non-asymptotic and asymptotic upper bounds on $\objadv$. In Section~\ref{sec:comp_complex}, we analyze the computational complexity of the proposed scheme and compare it against existing baselines. Finally, in Section~\ref{sec:spline_adversarial_tolerance}, we establish a complementary lower bound on the worst-case adversarial error of the proposed encoding and decoding design.

\subsection{Code Design}\label{sec:code_design}
In this subsection, we develop the encoder and decoder constructions
for the adversarial optimization problem in
\eqref{eq:main_adv_obj}. To state the proposed design and the
performance guarantees, we impose the following assumptions
throughout this section:

\begin{enumerate}[
    label=\textnormal{(A\arabic*)},
    ref=\textnormal{A\arabic*},
    leftmargin=3.8em
]

\item \label{assump:f_bounded}
\emph{Bounded computation outputs:}
The computing function
$f:\mathbb{R}^d\to\mathbb{R}^m$ satisfies
\begin{align}
    \sup_{\mathbf{x}\in\mathbb{R}^d}
    \norm{f(\mathbf{x})}_{\infty}
    \leq
    T_f
\end{align}
for some constant $T_f>0$.

\item \label{assump:f_regularity}
\emph{Regularity of the computation function:}
Let
\(
    f=[f_1,\ldots,f_m]^T.
\)
For every $\mathbf{x}\in\mathbb{R}^d$ and $j\in[m]$, assume that
\begin{align}
    \norm{\nabla f_j(\mathbf{x})}_2
    &\leq
    G_j,
    &
    \norm{D^2f_j(\mathbf{x})}_{\op}
    &\leq
    H_j.
\end{align}
We define the aggregate gradient and Hessian bounds as
\begin{align}
    G_f
    &:=
    \left(
        \sum_{j=1}^m G_j^2
    \right)^{1/2},
    &
    H_f
    &:=
    \left(
        \sum_{j=1}^m H_j^2
    \right)^{1/2},
\end{align}
and we assume that $G_f,H_f>0$.

\item \label{assump:adv_bounded}
\emph{Bounded adversarial outputs:}
Every corrupted output belongs to $[-T_a,T_a]^m$ for some
$T_a>0$. We define
\begin{align}
    T:=\max\{T_a,T_f\}.
\end{align}

\end{enumerate}

Directly optimizing \eqref{eq:main_adv_obj} is difficult because the
encoder and decoder belong to infinite-dimensional function spaces
and up to $q$ returned worker outputs may be arbitrarily corrupted. A
central step in our approach is to construct an upper bound on
$\Radv(\enc,\dec;f)$ that decomposes the loss into three interconnected
components.

The first component captures the encoder approximation error, which
measures the discrepancy caused by approximating the original input
data at the encoder design points, corresponding to
\eqref{eq:enc_approx}. The remaining two components capture the
decoder error, corresponding to \eqref{eq:dec_approx}, by separating
the effect of adversarial corruptions from the clean generalization
error of the decoder.

More formally, for any admissible attack $A\in\Aadv$, we have
\begin{align}
    &\frac{1}{K}
    \sum_{k=1}^K
    \norm{
        \fhat(\mathbf{x}_k)
        -
        f(\mathbf{x}_k)
    }_2^2
    \nonumber\\
    &=
    \frac{1}{K}
    \sum_{k=1}^K
    \norm{
        \dec(\alpha_k)
        -
        f(\enc(\alpha_k))
        +
        f(\enc(\alpha_k))
        -
        f(\mathbf{x}_k)
    }_2^2
    \nonumber\\
    &\leq
    \frac{2}{K}
    \sum_{k=1}^K
    \norm{
        f(\enc(\alpha_k))
        -
        f(\mathbf{x}_k)
    }_2^2+
    \frac{2}{K}
    \sum_{k=1}^K
    \norm{
        \dec(\alpha_k)
        -
        f(\enc(\alpha_k))
    }_2^2,
\end{align}
where the inequality follows from
$
    \norm{\mathbf{u}+\mathbf{v}}_2^2
    \leq
    2\norm{\mathbf{u}}_2^2
    +
    2\norm{\mathbf{v}}_2^2.
$
Taking the supremum over all admissible attacks yields
\begin{align}
    \Radv(\enc,\dec;f)
    \leq
    \lenc+\ldec,
\label{eq:adv_first_decomposition}
\end{align}
where
\begin{align}
    \lenc
    &:=
    \frac{2}{K}
    \sum_{k=1}^K
    \norm{
        f(\enc(\alpha_k))
        -
        f(\mathbf{x}_k)
    }_2^2,
    \label{eq:adv_encoding_term}\\
    \ldec
    &:=
    \sup_{A\in\Aadv}
    \frac{2}{K}
    \sum_{k=1}^K
    \norm{
        \dec(\alpha_k)
        -
        f(\enc(\alpha_k))
    }_2^2.
    \label{eq:adv_decoding_term}
\end{align}
The encoder term $\lenc$ is independent of the attack. In contrast,
the value of $\ldec$ is affected by the attack through the returned
worker outputs used to construct $\dec$. 

To systematically manage the worst-case decoding error, we further decompose the decoder error. Let $\decclean$ represent the \emph{clean decoder} function that would be constructed if all worker nodes were honest, i.e., if the master node received $\overline{\mathbf{y}}_j^{A} = f(\enc(\beta_j))$ for all $j \in [N]$. Thus, $\decclean$ is a
counterfactual function introduced only for the analysis. Adding and
subtracting $\decclean(\alpha_k)$ in
\eqref{eq:adv_decoding_term} gives
\begin{align}
    \ldec
    &=
    \sup_{A\in\Aadv}
    \frac{2}{K}
    \sum_{k=1}^K
    \norm{
        \dec(\alpha_k)
        -
        \decclean(\alpha_k)
        +
        \decclean(\alpha_k)
        -
        f(\enc(\alpha_k))
    }_2^2
    \nonumber\\
    &\le
    2 \sup_{A \in \Aadv} \sup_{z \in \Omega} \left\| \dec(z) - \decclean(z) + \decclean(z) - f(\enc(z)) \right\|_2^2 \nonumber\\
    &\le
    \underbrace{
    4 \sup_{A \in \Aadv} \sup_{z \in \Omega} \left\| \dec(z) - \decclean(z) \right\|_2^2
    }_{\ldecadv}
    +
    \underbrace{
    4 \sup_{z \in \Omega} \left\| \decclean(z) - f(\enc(z)) \right\|_2^2
    }_{\ldecgen}.
\label{eq:adv_decoder_decomposition}
\end{align}
Consequently,
\begin{align}\label{eq:adv_three_term_decomposition}
    \Radv(\enc,\dec;f)
    \leq
    \lenc
    +
    \ldecgen
    +
    \ldecadv.
\end{align}

\begin{remark}[Interpretation of the Error Terms]\label{rem:error_terms}
The term $\lenc$ represents the encoder approximation error, measuring the average distortion at the encoder design points $\{\alpha_i\}_{i=1}^K$ after applying $f$. The term $\ldecadv$ captures the \emph{adversarial perturbation error}, defined as the worst-case deviation in the decoder output induced by up to $q$ corrupted worker transmissions. Finally, $\ldecgen$ represents the \emph{clean decoder generalization error}, which reflects the reconstruction error of the composite map $z \mapsto f(\enc(z))$ in the benign setting where all worker outputs are uncorrupted.  
\end{remark}

\begin{remark}[Connection to Generalization in Learning Theory]\label{rem:generalization_analogy}
The term \emph{generalization error} is borrowed from learning theory, where it measures how well a model trained on one dataset performs on new, unseen test data~\cite{bach2024learning}. In our setting, the decoder is constructed based on the decoder design points $\{\beta_j\}_{j=1}^N$, whereas $\ldecgen$ measures its performance at the distinct encoder design points $\{\alpha_i\}_{i=1}^K$.
\end{remark}

\begin{remark}[Coupling through the encoder]\label{rem:adv_coupling}
Although the three terms in
\eqref{eq:adv_three_term_decomposition} have different
interpretations, they are not independent. The encoder term depends
directly on the values $\enc(\alpha_k)$, while both decoder terms
depend on the encoder through the composite function $f\circ\enc$.
\end{remark}

We address this coupling through a nested construction. For a fixed
encoder, we first design the decoder and derive bounds on $\ldecadv+\ldecgen$.
Substituting these
bounds into \eqref{eq:adv_three_term_decomposition} yields an
encoder-only objective, which is then minimized in an outer
optimization step.

\subsubsection{\bf Decoder Design}\label{sec:adv_decoder_design}

As part of the code design, we select the decoder design points to be
equidistant over $\Omega=[0,1]$:
\begin{align}\label{eq:equispaced_decoder_points}
    \beta_j
    :=
    \frac{j-1}{N-1},
    \qquad
    j\in[N].
\end{align}
This choice will help us in our subsequent theoretical results.

Fix an encoder $\enc\in\hild{2}{d}$. The master node receives the outputs $\{\overline{\mathbf{y}}_j^A\}_{j=1}^N$ defined in \eqref{eq:received_output}
without knowing which, if any, of the returned outputs are corrupted.
We construct the decoder using a regularized smoothing-spline objective. Specifically, the decoder is defined as:
\begin{align}\label{eq:decoder_opt_adv}
    \dec^*
    :=
    \operatorname*{arg\,min}_{u\in\hild{2}{m}}
    \left\{
        \frac{1}{N}
        \sum_{j=1}^N
        \norm{
            u(\beta_j)
            -
            \overline{\mathbf{y}}_j^{A}
        }_2^2
        +
        \declamb
        \norm{u''}_{\lpd{2}{m}}^2
    \right\},
\end{align}
where $\declamb>0$ is the decoder smoothing parameter. The smoothing parameter $\declamb$ may depend on the total number of workers $N$, the adversarial budget $q$, and the decoder evaluation points $\{\beta_j\}_{j=1}^N$, and is selected prior to observing the worker transmissions.

The first term in~\eqref{eq:decoder_opt_adv} measures the empirical squared error over all returned worker outputs, while the second term penalizes the curvature of the decoder. The smoothing parameter $\declamb$ balances these two objectives: larger values of $\declamb$ yield a smoother decoder with lower curvature, whereas smaller values prioritize fitting the empirical data, pushing the decoder toward interpolating the received outputs $\{\overline{\mathbf{y}}_j^A\}_{j=1}^N$. By penalizing rapid variations, this roughness penalty reduces sensitivity to localized anomalies and limits the impact of adversarial corruptions.

Based on decoder function in \eqref{eq:decoder_opt_adv}, its corresponding \emph{clean decoder} has the following form:
\begin{align}\label{eq:clean_decoder_opt}
    \decclean
    :=
    \operatorname*{arg\,min}_{u\in\hild{2}{m}}
    \left\{
        \frac{1}{N}
        \sum_{j=1}^N
        \norm{
            u(\beta_j)
            -
            f(\enc(\beta_j))
        }_2^2
        +
        \declamb
        \norm{u''}_{\lpd{2}{m}}^2
    \right\}.
\end{align}
The clean decoder which is purely an analytical tool used in our analysis, corresponds to the exact same optimization, assuming that all worker nodes are honest and compute and transmit uncorrupted evaluations.

As established in Section~\ref{sec:smooth_spline}, the separability of the squared Euclidean norm across coordinate dimensions decouples~\eqref{eq:decoder_opt_adv} and~\eqref{eq:clean_decoder_opt} into $m$ independent scalar smoothing-spline variational problems. Consequently, both $\dec$ and $\decclean$ are unique vector-valued natural cubic smoothing splines. Specifically, the decoder admits the representer decomposition
\begin{align}\label{eq:adv_decoder_spline_form}
    \dec^*(t)
    =
    \mathbf{d}_0
    +
    \mathbf{d}_1 t
    +
    \sum_{j=1}^N
    \mathbf{c}_j
    \phi_0(t, \beta_j),
\end{align}
where $\mathbf{d}_0, \mathbf{d}_1, \mathbf{c}_1, \ldots, \mathbf{c}_N \in \mathbb{R}^m$, and $\phi_0$ denotes the reproducing kernel associated with the space $\hilz{2}$, which is the second-order Sobolev space with compact support (see Definition~\ref{def:sob_compact} and Remark~\ref{rem:kernel} in Appendix~\ref{app:functional_details} for its explicit formulation). The coefficients are obtained by solving
the standard linear system for vector-valued cubic smoothing splines~\cite{wahba1975smoothing,wahba1990spline}.

\begin{theorem}[Decoder error bound]\label{th:decoder_error_bound}
Consider the \gcc{} scheme with $N$ worker nodes and the
equidistant decoder design points in
\eqref{eq:equispaced_decoder_points}. Suppose that
Assumptions~\ref{assump:f_bounded} through
\ref{assump:adv_bounded} hold. Then there exist constants
$C_1,C_2,c_0>0$, independent of $N$, $q$, $\declamb$, and $\enc$,
such that, whenever
\(
    0<\declamb\leq1
    \text{ and }
    N\declamb^{1/4}\geq c_0,
\)
we have
\begin{align}\label{eq:combined_decoder_bound}
    \inf_{\dec \in \hild{2}{m}}\ldecgen
    +
    \ldecadv
    &\leq
    C_1
    \max\{H_f^2,G_f^2\}
    (
        \declamb^{3/4}
        +
        N^{-3}
    )
    \psi(
        \norm{\enc}_{\hild{2}{d}}^2
    )
   +
    C_2
    mT^2
    \frac{q^2}{N^2}
    \declamb^{-1/2},
\end{align}
where
\(
    \psi(t):=t+4t^2.
\)
\end{theorem}

\begin{remark}[Trade-off in the Decoder Error Bound]
The first term in~\eqref{eq:combined_decoder_bound} upper-bounds the clean decoder generalization error. For a fixed number of workers $N$, this term decreases monotonically as $\declamb$ decreases. It also scales with the regularity of the computation pipeline, depending on the smoothness parameters $H_f$ and $G_f$ of the target function, as well as the encoder Sobolev norm $\norm{\enc}_{\hild{2}{d}}^2$ which means smoother functions and encoders yield a strictly smaller bound. The second term bounds the adversarial perturbation error. In contrast to the first term, it grows as $\declamb$ decreases or as the adversary's leverage increases through the corruption ratio $Tq/N$. Consequently, the regularization parameter $\declamb$ controls a fundamental trade-off between clean reconstruction accuracy and robustness to adversarial corruptions.
\end{remark}

\begin{remark}[Necessity of Positive Decoder Smoothing]
In the straggler setting, taking the interpolating limit $\declamb \to 0$ yields an effective parameter-free decoder~\cite{moradi2026learning}. In the adversarial regime, however, interpolation is catastrophic, as interpolating arbitrary corruptions severely distorts the reconstructed function. The  factor $\declamb^{-1/2}$ in~\eqref{eq:combined_decoder_bound} captures this instability explicitly, suggesting why strictly positive regularization ($\declamb > 0$) is essential in the presence of malicious worker nodes.
\end{remark}


\begin{corollary}[Optimal scaling of the decoder smoothing parameter]
\label{cor:optimal_decoder_smoothing}
Assume that the conditions of
Theorem~\ref{th:decoder_error_bound} hold.
Suppose that $q\geq1$. Then, for all sufficiently
large $N$, the value of $\declamb$ that minimizes the decoder-error
upper bound in \eqref{eq:combined_decoder_bound} satisfies
\begin{align}\label{eq:optimal_decoder_smoothing_scaling}
    \declamb^{\star}
    \asymp
    \left(
        \frac{q}{N}
    \right)^{8/5},
\end{align}
where the implicit constants may depend on the fixed encoder and the
constant parameters of the scheme, but are independent of $N$ and $q$.
\end{corollary}

\subsubsection{\bf Encoder Design}\label{sec:adv_encoder_design}

We next substitute the decoder bound into
\eqref{eq:adv_three_term_decomposition} to obtain an encoder-only
objective. Assumption~\ref{assump:f_regularity} implies that $f$ is
$G_f$-Lipschitz:
\begin{align}
    \norm{
        f(\mathbf{x})
        -
        f(\widetilde{\mathbf{x}})
    }_2
    \leq
    G_f
    \norm{
        \mathbf{x}
        -
        \widetilde{\mathbf{x}}
    }_2.
\end{align}
Therefore,
\begin{align}
    \lenc
    &=
    \frac{2}{K}
    \sum_{k=1}^K
    \norm{
        f(\enc(\alpha_k))
        -
        f(\mathbf{x}_k)
    }_2^2
    \nonumber\\
    &\leq
    \frac{2G_f^2}{K}
    \sum_{k=1}^K
    \norm{
        \enc(\alpha_k)
        -
        \mathbf{x}_k
    }_2^2.
\label{eq:adv_encoder_error_bound}
\end{align}
Combining \eqref{eq:adv_three_term_decomposition},
\eqref{eq:adv_encoder_error_bound}, and
Theorem~\ref{th:decoder_error_bound} gives
\begin{align}
    \Radv(\enc,\decstar;f)
    &\leq
    \frac{2G_f^2}{K}
    \sum_{k=1}^K
    \norm{
        \enc(\alpha_k)
        -
        \mathbf{x}_k
    }_2^2
    \nonumber\\
    &\quad+
    C_1
    \max\{H_f^2,G_f^2\}
    \left(
        \declamb^{3/4}
        +
        N^{-3}
    \right)
    \psi\left(
        \norm{\enc}_{\hild{2}{d}}^2
    \right)
    \nonumber\\
    &\quad+
    C_2
    mT^2
    \frac{q^2}{N^2}
    \declamb^{-1/2}.
\label{eq:adv_unified_encoder_bound}
\end{align}

The last term in \eqref{eq:adv_unified_encoder_bound} is independent
of the encoder. Assuming $G_f>0$, define
\begin{align}\label{eq:adv_encoder_regularization}
    \Gamma_{q,N}
    &:=
    \frac{C_1}{2G_f^2}
    \max\{H_f^2,G_f^2\}
    \left(
        \declamb^{3/4}
        +
        N^{-3}
    \right).
\end{align}
The dependence of $\Gamma_{q,N}$ on the adversarial budget $q$ is
through the choice of the decoder smoothing parameter $\declamb$.

Ignoring the additive term that is independent of the encoder, we
choose the encoder as
\begin{align}\label{eq:adv_encoder_opt}
    \encstar
    :=
    \operatorname*{arg\,min}_{u\in\hild{2}{d}}
    \left\{
        \frac{1}{K}
        \sum_{k=1}^K
        \norm{
            u(\alpha_k)
            -
            \mathbf{x}_k
        }_2^2
        +
        \Gamma_{q,N}
        \psi\left(
            \norm{u}_{\hild{2}{d}}^2
        \right)
    \right\}.
\end{align}

Since $\psi$ is strictly increasing, the representer theorem stated
in Theorem~\ref{thm:representer} implies that every minimizer of
\eqref{eq:adv_encoder_opt} admits the finite-dimensional
representation
\begin{align}\label{eq:adv_encoder_representation}
    \encstar(t)
    =
    \sum_{k=1}^K
    \phi(t,\alpha_k)\mathbf{z}_k,
\end{align}
for some coefficient vectors
$\mathbf{z}_1,\ldots,\mathbf{z}_K\in\mathbb{R}^d$, where
$\phi$ is the scalar reproducing kernel of $\hild{2}{}$ given by
\(
    \phi(t,s)
    =
    1+ts+\phi_0(t,s),
\)
with $\phi_0$ defined in Remark~\ref{rem:kernel}. Thus, the
infinite-dimensional encoder design reduces to a finite-dimensional
optimization over the coefficient vectors
$\{\mathbf{z}_k\}_{k=1}^K$.

The encoder objective in \eqref{eq:adv_encoder_opt} contains the
nonlinear regularization term
\(
    \psi\bigl(\norm{u}_{\hild{2}{d}}^2\bigr)
\),
where $\psi(t)=t+4t^2$. Although
\eqref{eq:adv_encoder_representation} reduces the problem to a
finite-dimensional optimization, directly computing its coefficient
vectors can still be computationally demanding. As in the straggler
setting studied in~\cite{moradi2026learning}, the encoder can instead
be constructed through a standard cubic smoothing-spline problem.
The following theorem describes this reduction for the
adversarial setting.

For convenience, define the encoder-independent adversarial term
\begin{align}\label{eq:encoder_independent_adv_term}
    \Delta_{q,N}
    :=
    C_2
    mT^2
    \frac{q^2}{N^2}
    \declamb^{-1/2}.
\end{align}

\begin{theorem}[Efficient encoder construction]
\label{th:efficient_adv_encoder}
Suppose that Assumptions~\ref{assump:f_bounded} through
\ref{assump:adv_bounded} and the conditions of
Theorem~\ref{th:decoder_error_bound} hold. There exist constants
$m_1,m_2>0$ and
\(
    R_{\mathrm e}
    =
    R_{\mathrm e}
    \left(
        \{\alpha_k\}_{k=1}^K,
        \{\mathbf{x}_k\}_{k=1}^K
    \right)
    >0
\)
such that
\begin{align}\label{eq:efficient_adv_encoder_loss_bound}
    \objadv
    &\leq
    \Delta_{q,N}
    +
    2G_f^2
    \bigg[
        \frac{1}{K}
        \sum_{k=1}^K
        \norm{
            \encdagger(\alpha_k)
            -
            \mathbf{x}_k
        }_2^2
        +
        \enclamb
        \left(
            R_{\mathrm e}
            +
            \norm{
                (\encdagger)''
            }_{\lpd{2}{d}}^2
        \right)
    \bigg],
\end{align}
where the encoder smoothing parameter is
\begin{align}\label{eq:enc_sm_param}
    \enclamb
    :=
    \Gamma_{q,N}
    \left(
        m_1+m_2R_{\mathrm e}
    \right),
\end{align}
and $\encdagger$ is a minimizer of
\begin{align}\label{eq:efficient_adv_encoder_opt}
    \encdagger
    :=
    \operatorname*{arg\,min}_{u\in\hild{2}{d}}
    \left\{
        \frac{1}{K}
        \sum_{k=1}^K
        \norm{
            u(\alpha_k)
            -
            \mathbf{x}_k
        }_2^2
        +
        \enclamb
        \norm{u''}_{\lpd{2}{d}}^2
    \right\}.
\end{align}
\end{theorem}

The proof follows the same argument used for the encoder construction in the straggler
setting~\cite{moradi2026learning}. The only changes are the
regularization coefficient $\Gamma_{q,N}$ and the additive term
$\Delta_{q,N}$, which is independent of the encoder and therefore
does not affect the encoder optimization.

The squared Euclidean loss and the second-derivative penalty in
\eqref{eq:efficient_adv_encoder_opt} decouple across the $d$ input
coordinates. Consequently, $\encdagger$ is a vector-valued natural
cubic smoothing spline and admits the representation
\begin{align}\label{eq:adv_encoder_final_representation}
    \encdagger(t)
    =
    \mathbf{a}_0
    +
    \mathbf{a}_1t
    +
    \sum_{k=1}^K
    \mathbf{b}_k
    \phi_0(t,\alpha_k),
\end{align}
where
\(
    \mathbf{a}_0,\mathbf{a}_1,
    \mathbf{b}_1,\ldots,\mathbf{b}_K
    \in\mathbb{R}^d
\), and $\phi_0(\cdot)$ defined in Remark~\ref{rem:kernel}. Therefore, the encoder coefficients can be computed efficiently by
solving the standard linear system for vector-valued cubic smoothing
splines with smoothing parameter $\enclamb$.

\begin{algorithm}[t]
\caption{General Coded Computing in the Adversarial Setting}
\label{alg:adv_gcc}
\begin{algorithmic}[1]

\Require Input batch $\{\mathbf{x}_k\}_{k=1}^K$, target function
$f$, number of workers $N$, adversarial budget $q$, encoder design
points $\{\alpha_k\}_{k=1}^K$, and smoothing parameters
$\enclamb$ and $\declamb$.

\Ensure Estimates
$\{\fhat(\mathbf{x}_k)\}_{k=1}^K$.

\State The master node selects the equidistant decoder design points
\(
    \beta_j
    =
    \frac{j-1}{N-1}
    \text{ for }
    j\in[N].
\)

\State The master node constructs the encoder $\encdagger$ by solving
\[
    \encdagger
    =
    \operatorname*{arg\,min}_{u\in\hild{2}{d}}
    \left\{
        \frac{1}{K}
        \sum_{k=1}^K
        \norm{
            u(\alpha_k)
            -
            \mathbf{x}_k
        }_2^2
        +
        \enclamb
        \norm{u''}_{\lpd{2}{d}}^2
    \right\},
\]
which admits the vector-valued cubic smoothing-spline representation
in \eqref{eq:adv_encoder_final_representation}.

\State The master node computes the coded inputs
\(
    \widetilde{\mathbf{x}}_j
    =
    \encdagger(\beta_j)
    \text{ for }
    j\in[N],
\)
and sends $\widetilde{\mathbf{x}}_j$ to worker $j$.

\State Each honest worker $j$ computes and returns
\(
    f(\widetilde{\mathbf{x}}_j)
    =
    f(\encdagger(\beta_j)).
\)
An adversarial worker may instead return an arbitrary vector
$\mathbf{a}_j\in[-T_a,T_a]^m$.

\State The master node receives
$\{\overline{\mathbf{y}}_j^A\}_{j=1}^N$ and constructs the decoder
from all returned outputs by solving
\[
    \decstar
    =
    \operatorname*{arg\,min}_{u\in\hild{2}{m}}
    \left\{
        \frac{1}{N}
        \sum_{j=1}^N
        \norm{
            u(\beta_j)
            -
            \overline{\mathbf{y}}_j^A
        }_2^2
        +
        \declamb
        \norm{u''}_{\lpd{2}{m}}^2
    \right\},
\]
which admits the vector-valued cubic smoothing-spline representation
in \eqref{eq:adv_decoder_spline_form}.

\State The master node outputs
\(
    \fhat(\mathbf{x}_k)
    =
    \decstar(\alpha_k)\text{ for }
    k\in[K].
\)

\end{algorithmic}
\end{algorithm}

\subsubsection{\gcc{} Scheme}

The preceding constructions define the \gcc{} scheme in the adversarial setting.
The master node first constructs the encoder using
\eqref{eq:adv_encoder_opt} (or using \eqref{eq:efficient_adv_encoder_opt} for efficient encoder), evaluates it at the equidistant decoder
design points in \eqref{eq:equispaced_decoder_points}, and sends the
resulting coded inputs to the workers. After receiving all outputs from worker nodes, the master node constructs the decoder in
\eqref{eq:decoder_opt_adv} using all returned outputs without
attempting to identify the adversarial workers.  using all
returned outputs, without attempting to identify the adversarial workers. Finally, it computes
the estimates
\(
    \fhat(\mathbf{x}_k)
    =
    \dec(\alpha_k)
    \text{ for }
    k\in[K]
\). Algorithm~\ref{alg:adv_gcc} summarizes the efficient
implementation of the proposed scheme.

In terms of asymptotic analysis, the following corollary establishes the asymptotic convergence rate of the robust end-to-end objective as the worker nodes increases, demonstrating that the adversarial risk in~\eqref{eq:main_adv_obj} asymptotically vanishes whenever the adversarial budget satisfies $q/N \to 0$ as $N \to \infty$.

\begin{corollary}[Asymptotic Convergence Rate]
\label{cor:total_convergence_rate}
Under the assumptions of Theorem~\ref{th:decoder_error_bound}, suppose the adversarial budget satisfies
\(
    1 
    \leq q \le C_q N^\theta
\)
for constants $C_q > 0$ and $\theta \in [0, 1)$. If the decoder smoothing parameter is chosen according to Corollary~\ref{cor:optimal_decoder_smoothing} such that $\declamb \asymp (q/N)^{8/5}$, then for sufficiently large $N$, the end-to-end objective satisfies
\begin{align}\label{eq:total_convergence_rate_power}
    \objadv
    =
    \mathcal{O}\left(
        N^{\frac{6}{5}(\theta-1)}
    \right).
\end{align}
\end{corollary}

\subsection{Computational Complexity}\label{sec:comp_complex}

The computational complexity of the encoder and decoder constructions
closely follows the analysis developed for the straggler setting
in~\cite{moradi2026learning}. We therefore summarize the main results
and highlight the difference introduced by the adversarial model.

For the nonlinear encoder in \eqref{eq:adv_encoder_opt}, the
finite-dimensional kernel coefficients can be computed using an
iterative optimization method such as limited-memory Broyden--Fletcher--Goldfarb--Shanno (L-BFGS)~\cite{nocedal2006numerical}. For L-BFGC solver, If the solver performs $I$ iterations,
constructing the encoder and evaluating it at the $N$ decoder design
points require
\begin{align}
    \mathcal{O}\bigl(IK^2d+NKd\bigr)
\end{align}
operations. In contrast, the efficient encoder in
\eqref{eq:efficient_adv_encoder_opt} is a vector-valued cubic
smoothing spline. Using a B-spline implementation with a banded linear
system, it can be constructed and evaluated at the decoder design
points in
\begin{align}
    \mathcal{O}\bigl((K+N)d\bigr)
\end{align}
operations.

The decoder in \eqref{eq:decoder_opt_adv} has the same smoothing-spline
structure. Unlike the straggler setting, all $N$ worker nodes return their
results, and hence the decoder is constructed using all $N$ received
outputs. Constructing the decoder and evaluating it at the $K$ encoder
design points require
\begin{align}
    \mathcal{O}\bigl((N+K)m\bigr)
\end{align}
operations. Consequently, with the efficient encoder, the total master-side
computational complexity of the adversarial \gcc{} scheme is
\begin{align}\label{eq:adv_gcc_overhead}
    \mathcal{O}\bigl((K+N)d\bigr)
    +
    \mathcal{O}\bigl((N+K)m\bigr).
\end{align}
In the common regime $N\geq K$, this reduces to
$\mathcal{O}\bigl(N(d+m)\bigr)$.

\subsubsection{Comparison with Existing Schemes}

We compare the computational complexity of different coded computing schemes at the master node, excluding
worker-side evaluations of the target function.

Lagrange Coded Computing (LCC) relies on polynomial interpolation and fast multipoint evaluation~\cite{yu2019lagrange, borodin1974fast}, yielding encoding and decoding complexities of
\(
    \mathcal{O}(N \log^2 K \log\log K \cdot d)
\)
and
\(
    \mathcal{O}(N \log^2 N \log\log N \cdot m)
\),
respectively. In the presence of adversarial workers, LCC requires an explicit error-localization step to identify corrupted outputs. By leveraging sub-quadratic decoding routines for Reed--Solomon codes~\cite{jl1967shift, berlekamp1966nonbinary}, locating these errors incurs an additional complexity of $\mathcal{O}(q \log^2 q \log\log q \cdot m)$. Consequently, in addition to the overhead of error localization, LCC exhibits superlinear computational complexity with respect to the number of worker nodes $N$. Overall, for an input batch size $K$ and at most $q \leq N$ adversaries, the total computational complexity at the master node is
\begin{align}
    &\mathcal{O}\!\left(N\log^2 K\log\log K\cdot d\right)
    + \mathcal{O}\!\left(N\log^2 N\log\log N\cdot m\right)
    + \mathcal{O}\!\left(q\log^2 q\log\log q\cdot m\right) \nonumber \\
    &\quad = \mathcal{O}\!\left(N\log^2 K\log\log K\cdot d + N\log^2 N\log\log N\cdot m\right).
\end{align}


Berrut Approximated Coded Computing (BACC) relies on barycentric rational interpolation~\cite{jahani2022berrut}, where direct evaluation yields encoding and decoding complexities of $\mathcal{O}(KNd)$ and $\mathcal{O}(KNm)$, respectively. Because BACC alone cannot identify corrupted worker outputs, ApproxIFER~\cite{soleymani2022approxifer} extends it to the adversarial setting by adapting the Berlekamp--Welch approach~\cite{berlekamp1966nonbinary} to rational interpolation. Specifically, before performing the final rational reconstruction, ApproxIFER constructs candidate error-locator functions by solving linear systems of size $\mathcal{O}(N)$ across the $m$ output coordinates and combining the coordinatewise estimates to filter out suspected adversarial reports. Although~\cite{soleymani2022approxifer} does not explicitly analyze the asymptotic cost of this error-localization step, solving these $\mathcal{O}(N)$-dimensional linear systems requires $\mathcal{O}(mN^3)$ operations under standard dense solvers.


In contrast, our proposed \gcc{} framework requires only $\mathcal{O}((K+N)d)$ and $\mathcal{O}((K+N)m)$ arithmetic operations for encoding and decoding, respectively, yielding a total master-node complexity of
\begin{align}
    \mathcal{O}\!\left((K + N)(d + m)\right).
\end{align}
Compared with BACC-based schemes, \gcc{} reduces the multiplicative $\mathcal{O}(KN)$ dependence on the input batch size $K$ and worker count $N$ to a purely additive $\mathcal{O}(K + N)$ scaling. Furthermore, unlike LCC, Robust ALCC, ApproxIFER, and RBACC, \gcc{} mitigates corrupted worker reports inherently through spline regularization at the decoder, completely eliminating the need for dedicated error-detection or error-localization stages. Consequently, the computational complexity of \gcc{} remains strictly linear in $K + N$ and is entirely independent of the adversarial budget $q$.

\subsection{Impossibility Result for Smoothing-Spline Decoding}
\label{sec:spline_adversarial_tolerance}

Corollary~\ref{cor:total_convergence_rate} establishes vanishing worst-case error for the proposed scheme when the adversarial budget grows sublinearly with the number of workers ($q = o(N)$). In the following theorem, we show that this sublinear scaling is necessary for the smoothing-spline decoder in~\eqref{eq:decoder_opt_adv} to achieve vanishing adversarial loss.

\begin{theorem}[Adversarial Error Lower Bound]
\label{th:spline_impos}
Consider a \gcc{} scheme with $N \geq 2$, $K \geq 1$, and an adversarial budget $1 \leq q < N$. For any encoder $\enc \in \hild{2}{d}$, let $\decstar$ denote the smoothing-spline decoder defined in~\eqref{eq:decoder_opt_adv} with smoothing parameter $\declamb > 0$. Under Assumption~\textnormal{(A3)} with $T_a > 0$, for every computing function $f \in \dmsmfunc$, we have
\begin{align}\label{eq:spline_converse_finite}
    \inf_{\substack{
        \enc \in \hild{2}{d}, \\
        \declamb > 0
    }}
    \Radv(\enc, \decstar; f)
    \geq
    \frac{T_a^2}{K}
    \left(\frac{q}{N}\right)^2.
\end{align}
\end{theorem}

\begin{corollary}[Positive Adversarial Fraction]
\label{cor:spline_linear_adversarial_fraction}
Fix $K \geq 1$, $\mu \in (0, 1)$, and $T_a > 0$, and let $q = \lfloor \mu N \rfloor$. Then, for every computing function $f \in \dmsmfunc$,
\begin{align}
    \liminf_{N \to \infty}
    \inf_{\substack{
        \enc \in \hild{2}{d}, \\
        \declamb > 0
    }}
    \Radv(\enc, \decstar; f)
    \geq
    \frac{T_a^2 \mu^2}{K}
    > 0.
\end{align}
In particular, the worst-case approximation error of the smoothing-spline decoder cannot vanish as $N \to \infty$ when the adversarial budget constitutes a constant positive fraction of $N$.
\end{corollary}
\section{Proofs of the Main Results}\label{sec:proofs}

In this section, we provide the proofs of the main theorems and corollaries, each presented in a separate subsection. Several proofs rely on auxiliary lemmas whose proofs are deferred to Appendix~\ref{app:lemma_proof}.

\subsection{Proof of Theorem~\ref{th:decoder_error_bound}}
\label{sec:proof_decoder_error_bound}

Fix $\enc \in \hild{2}{d}$, and define
\begin{align}
    g(z)
    :=
    f(\enc(z)).
\end{align}

\begin{lemma}[{\cite[Lemma~1]{moradi2026learning}}]\label{lem:composition_sobolev}
Let $\enc \in \hild{2}{d}$, and let $f = [f_1, \ldots, f_m]^\top : \mathbb{R}^d \to \mathbb{R}^m$ satisfy Assumption~(A2). Then, $f \circ \enc \in \hild{2}{m}$.
\end{lemma}

Applying Lemma~\ref{lem:composition_sobolev}, we have $g \in \hild{2}{m}$. The clean decoder in \eqref{eq:clean_decoder_opt} is a smoothing spline constructed over the design points $\{\beta_j\}_{j=1}^N$ with respect to the target function $g(\cdot)$. Therefore, using the smoothing spline operator defined in \eqref{eq:smoothing_spline}, we obtain
\begin{align}\label{eq:clean_decoder_operator_form}
    \decclean
    =
    \smop{N}{\declamb}{g}.
\end{align}

\subsubsection{Upper bound for clean generalization error}
We first consider the clean generalization error. By definition in \eqref{eq:adv_decoder_decomposition}, we have
\begin{align}\label{eq:proof_gen_start}
    \ldecgen
    =
    4
    \sup_{z\in\Omega}
    \norm{
        \smop{N}{\declamb}{g}(z)
        -
        g(z)
    }_2^2.
\end{align}
Thus, we need a uniform upper bound for the second order smoothing-spline
approximation error. The following lemma provides the required estimate.

\begin{lemma}[Second-order smoothing-spline error bound]
\label{lem:clean_spline_approximation}
Let $h\in\hild{2}{m}$ and let
\(
    t_j=(j-1)/(n-1)
\)
for $j\in[n]$. Let $\smop{n}{\lambda}{h}$ be the second order smoothing spline constructed over design point set $\{t_j\}_{j\in[n]}$ and target function $h(\cdot)$ with smoothing parameter $0<\lambda \leq1$. Then, there exists constant
$C_{\mathrm{sp}}$, independent of $n$, $\lambda$, and $h$, such that
\begin{align}\label{eq:uniform_spline_approximation}
    \sup_{z\in\Omega}
    \norm{
        \smop{N}{\lambda}{h}(z)-h(z)
    }_2^2
    \leq
    C_{\mathrm{sp}}
    \left(
        \lambda^{3/4}
        +
        N^{-3}
    \right)
    \norm{h''}_{\lpd{2}{m}}^2.
\end{align}
\end{lemma}

The proof is provided in Appendix~\ref{proof_lem:clean_spline_approximation}. Applying Lemma~\ref{lem:clean_spline_approximation} to
$h=g=f\circ\enc$ and $\lambda=\declamb$ gives
\begin{align}\label{eq:proof_gen_after_spline}
    \ldecgen
    \leq
    4C_{\mathrm{sp}}
    \left(
        \declamb^{3/4}
        +
        N^{-3}
    \right)
    \norm{g''}_{\lpd{2}{m}}^2.
\end{align}
It remains to derive an upper bound for  the Sobolev regularity of the composite function $\norm{g''}_{\lpd{2}{m}}^2$. Following lemma provides an upper bound for $\norm{g''}_{\lpd{2}{m}}^2$ in terms of Sobolev norm of encoder $\enc$ and target function regularity constants $H_f, G_f$.

\begin{lemma}[Regularity of the composite function]
\label{lem:composition_regularity}
Under Assumption~(\textnormal{A}2), there exists a constant
$C_{\mathrm c}>0$ such that, for every
$u\in\hild{2}{d}$,
\begin{align}\label{eq:composition_regularity_bound}
    \norm{
        (f\circ u)''
    }_{\lpd{2}{m}}^2
    \leq
    C_{\mathrm c}
    \max\{H_f^2,G_f^2\}
    \psi\left(
        \norm{u}_{\hild{2}{d}}^2
    \right),
\end{align}
where
\(
    \psi(t)=t+4t^2.
\)
\end{lemma}

The detailed proof can be found in Appendix~\ref{proof_lem:composition_regularity}. Applying Lemma~\ref{lem:composition_regularity} to $u=\enc$ in
\eqref{eq:proof_gen_after_spline} yields
\begin{align}
    \ldecgen
    &\leq
    C_1
    \max\{H_f^2,G_f^2\}
    \left(
        \declamb^{3/4}
        +
        N^{-3}
    \right)
    \psi\left(
        \norm{\enc}_{\hild{2}{d}}^2
    \right),
\label{eq:proof_clean_decoder_error}
\end{align}
where
\(
    C_1:=4C_{\mathrm{sp}}C_{\mathrm c}.
\)

\subsubsection{Upper bound for Adversarial perturbation error}

We next bound the adversarial perturbation error. By linearity of the
smoothing-spline operator discussed in Section~\ref{sec:smooth_spline}, the corrupted and clean decoders admit the
representations
\begin{align}
    \decstar(z)
    &=
    \frac{1}{N}
    \sum_{j=1}^N
    G_{N,\declamb}(z,\beta_j)
    \overline{\mathbf{y}}_j^A,
    \label{eq:proof_corrupted_weight_representation}\\
    \decclean(z)
    &=
    \frac{1}{N}
    \sum_{j=1}^N
    G_{N,\declamb}(z,\beta_j)
    g(\beta_j).
    \label{eq:proof_clean_weight_representation}
\end{align}
Let $\mathcal{B}$ denote the adversarial worker set. Since the outputs
of all workers outside $\mathcal{B}$ are uncorrupted, subtracting
\eqref{eq:proof_clean_weight_representation} from
\eqref{eq:proof_corrupted_weight_representation} gives
\begin{align}\label{eq:proof_decoder_perturbation}
    \decstar(z)-\decclean(z)
    =
    \frac{1}{N}
    \sum_{j\in\mathcal{B}}
    G_{N,\declamb}(z,\beta_j)
    \left(
        \mathbf{a}_j-g(\beta_j)
    \right).
\end{align}

To bound \eqref{eq:proof_decoder_perturbation}, we require an upper bound on the exact discrete weight function $G_{N,\declamb}$, which does not follow directly from its definition. Instead, we derive it by comparing $G_{N,\declamb}$ with the continuous Green's function and its equivalent kernel. More specifically, for $h \in \lpd{2}{}$, consider the continuous smoothing problem
\begin{align}\label{eq:continuous_smoothing_problem}
    v_{\declamb}[h]
    :=
    \operatorname*{arg\,min}_{v\in\hil{2}}
    \left\{
        \int_\Omega
        \bigl(
            h(t)-v(t)
        \bigr)^2\,dt
        +
        \declamb
        \int_\Omega
        \bigl(
            v''(t)
        \bigr)^2\,dt
    \right\}.
\end{align}
It can be shown that the solution to \eqref{eq:continuous_smoothing_problem} is unique and can be represented using the continuous Green's function~\cite{roach1982green,cox1984multivariate}
$G_{\declamb}$ as
\begin{align}
    v_{\declamb}(x)
    =
    \int_\Omega
    G_{\declamb}(x,t)h(t)\,dt.
\end{align}

Similar to smoothing spline weight function $G_{N,\declamb}$, deriving  explicit formula or upper bound for $G_{\declamb}$ is challenging in general form. However, authors in \cite{messer1993new} proposed a kernel that is a good approximation of $G_{\declamb}$ for equidistant design points. Formally, define
\begin{align}\label{eq:proof_equivalent_kernel}
    K_{\declamb}(x,t)
    &=
    \frac{
        e^{-\frac{|x-t|}
        {\sqrt{2}{\declamb}^{1/4}}}
    }{
        2\sqrt{2}{\declamb}^{1/4}
    }
    \left(
        \sin\left(
            \frac{|x-t|}
            {\sqrt{2}{\declamb}^{1/4}}
        \right)
        +
        \cos\left(
            \frac{x-t}
            {\sqrt{2}{\declamb}^{1/4}}
        \right)
    \right)
    \nonumber\\
    &\quad+
    \frac{1}{
        2\sqrt{2}{\declamb}^{1/4}
    }
    \left[
        \Phi\left(
            \frac{x+t}{\sqrt{2}{\declamb}^{1/4}},
            \frac{x-t}{\sqrt{2}{\declamb}^{1/4}}
        \right)
    +
        \Phi\left(
            \frac{2-x-t}{\sqrt{2}{\declamb}^{1/4}},
            \frac{t-x}{\sqrt{2}{\declamb}^{1/4}}
        \right)
    \right],
\end{align}
where
\(
    \Phi(u,v)
    :=
    e^{-u}
    \left(
        \cos u-\sin u+2\cos v
    \right).
\)
We first bound this explicit kernel in the following lemma.

\begin{lemma}[Upper Bound on the Equivalent Kernel]
\label{lem:kernel_upperbound}
There exists a constant $\tau > 0$ such that for every $\lambda \in (0, 1]$,
\begin{align}\label{eq:kernel_upperbound}
    \sup_{x,t\in[0,1]}
    \abs{
        K_\lambda(x,t)
    }
    \leq
    \tau \lambda^{-1/4}.
\end{align}
\end{lemma}

See Appendix~\ref{proof_lem:kernel_upperbound} for the proof. Next, the following lemma relates the finite-sample spline weight function $G_{N,\lambda}$ and the equivalent kernel $K_\lambda$.

\begin{lemma}[Approximation of the Finite-Sample Spline Weights]
\label{lem:kernel_weight_distance}
Under equidistant design points, there exist constants $c_0, C_W, \kappa, C_K > 0$ such that whenever $0 < \lambda \leq 1$ and $N\lambda^{1/4} > c_0$, we have for all $x, t \in [0,1]$,
\begin{align}\label{eq:kernel_weight_distance}
    \abs{
        G_{N,\lambda}(x,t) - K_\lambda(x,t)
    }
    &\leq
    C_W\lambda^{-1/4}
    \exp\left(
        -\kappa
        \abs{x-t}
        \lambda^{-1/4}
    \right)
    +
    C_M
    \lambda^{-1/4}
    \exp\left(
        -\frac{1}{\sqrt{2}}
        \lambda^{-1/4}
    \right).
\end{align}
\end{lemma}

The detailed proof is provided in Appendix~\ref{proof_lem:kernel_weight_distance}. We can now derive an upper bound on the exact weight function used by
the decoder. By the triangle inequality,
\begin{align}
    \abs{G_{N,\declamb}(x,t)}
    &\leq
    \abs{K_{\declamb}(x,t)}
    +
    \abs{
        G_{N,\declamb}(x,t)-K_{\declamb}(x,t)
    }.
\end{align}
Using Lemmas~\ref{lem:kernel_upperbound} and
\ref{lem:kernel_weight_distance}, we obtain
\begin{align}
    \abs{G_{N,\declamb}(x,t)}
    &\leq
    \tau\declamb^{-1/4}
    +
    C_W\declamb^{-1/4}
    \exp\left(
        -\kappa|x-t|\declamb^{-1/4}
    \right)
    +
    C_M\declamb^{-1/4}
    \exp\left(
        -\frac{1}{\sqrt{2}}
        \declamb^{-1/4}
    \right)
    \nonumber\\
    &\lec{(a)}{\leq}
    \left(\tau+C_M +C_W\right)\declamb^{-1/4},
\end{align}
where (a) follows from the fact that both exponential terms are
bounded by one.

Therefore, defining 
\(
    C_{\mathrm G}:=\tau+C_M+C_W,
\)
we have 
\begin{align}\label{eq:weight_function_bound}
    \abs{G_{N,\declamb}(x,t)} \leq C_G \declamb^{-1/4}
\end{align}

We now return to the decoder adversarial perturbation error in
\eqref{eq:proof_decoder_perturbation}. Assumptions
\ref{assump:f_bounded} and \ref{assump:adv_bounded} imply
\begin{align}
    \norm{\mathbf{a}_j}_2
    &\leq
    \sqrt{m}\,T_a,
    \\
    \norm{g(\beta_j)}_2
    &=
    \norm{f(\enc(\beta_j))}_2
    \leq
    \sqrt{m}\,T_f.
\end{align}
Since $T=\max\{T_a,T_f\}$,
\begin{align}\label{eq:proof_corruption_bound}
    \norm{
        \mathbf{a}_j-g(\beta_j)
    }_2
    \leq
    2\sqrt{m}\,T.
\end{align}
Using \eqref{eq:weight_function_bound} and
\eqref{eq:proof_corruption_bound} in
\eqref{eq:proof_decoder_perturbation}, and
$|\mathcal{B}|\leq q$, gives
\begin{align}
    \sup_{z\in\Omega}
    \norm{
        \decstar(z)-\decclean(z)
    }_2
    &\leq
    \frac{2\sqrt{m}\,T}{N}
    \sum_{j\in\mathcal{B}}
    \sup_{z\in\Omega}
    \abs{
        G_{N,\declamb}(z,\beta_j)
    }
    \nonumber\\
    &\leq
    2C_{\mathrm G}
    \sqrt{m}\,T
    \frac{q}{N}
    \declamb^{-1/4}.
\end{align}
Consequently,
\begin{align}
    \ldecadv
    =
    4
    \sup_{A\in\Aadv}
    \sup_{z\in\Omega}
    \norm{
        \decstar(z)-\decclean(z)
    }_2^2
    \leq
    16C_{\mathrm G}^2
    mT^2
    \frac{q^2}{N^2}
    \declamb^{-1/2}.
    \label{eq:dec_adversarial_pert_final}
\end{align}
Thus, the adversarial perturbation term in
\eqref{eq:combined_decoder_bound} holds with
\begin{align}
    C_2
    :=
    16C_{\mathrm G}^2.
\end{align}

Adding \eqref{eq:proof_clean_decoder_error} and
\eqref{eq:dec_adversarial_pert_final} proves the claimed bound
for the constructed decoder $\decstar$ defined in \eqref{eq:decoder_opt_adv}. Since $\decstar$ is an
admissible decoder in $\hild{2}{m}$, we obtain
\begin{align}
    \inf_{\dec\in\hild{2}{m}}
    \left\{
        \ldecgen
        +
        \ldecadv
    \right\}
    \leq
    C_1
    \max\{H_f^2,G_f^2\}
    \left(
        \declamb^{3/4}
        +
        N^{-3}
    \right)
    \psi\left(
        \norm{\enc}_{\hild{2}{d}}^2
    \right)
    +
    C_2
    mT^2
    \frac{q^2}{N^2}
    \declamb^{-1/2}.
\end{align}
This proves the theorem.

\subsection{Proof of Corollary~\ref{cor:optimal_decoder_smoothing}}
\label{sec:proof_optimal_decoder_smoothing}

Define
\begin{align}
    A_{\enc}
    &:=
    C_1
    \max\{H_f^2,G_f^2\}
    \psi\left(
        \norm{\enc}_{\hild{2}{d}}^2
    \right),
    \\
    B_N
    &:=
    C_2mT^2
    \frac{q^2}{N^2}.
\end{align}
In the nondegenerate case where $q\geq1$ and $A_{\enc}>0$, the part of
the upper bound in \eqref{eq:combined_decoder_bound} that depends on
$\declamb$ is
\begin{align}
    \Phi_N(\declamb)
    :=
    A_{\enc}\declamb^{3/4}
    +
    B_N\declamb^{-1/2}.
\end{align}
Its derivative is
\begin{align}
    \Phi_N'(\declamb)
    =
    \frac{3}{4}
    A_{\enc}\declamb^{-1/4}
    -
    \frac{1}{2}
    B_N\declamb^{-3/2}.
\end{align}
The unique stationary point satisfies
\begin{align}
    \declamb^{5/4}
    =
    \frac{2B_N}{3A_{\enc}}.
\end{align}
Since $\Phi_N'$ is negative below this point and positive above it,
the stationary point is the unique minimizer:
\begin{align}
    \declamb^*
    =
    \left(
        \frac{
            2C_2mT^2q^2
        }{
            3C_1N^2
            \max\{H_f^2,G_f^2\}
            \psi\left(
                \norm{\enc}_{\hild{2}{d}}^2
            \right)
        }
    \right)^{4/5}.
\end{align}

For fixed $q$ and fixed problem parameters,
$\declamb^*\to0$ as $N\to\infty$. Moreover,
\begin{align}
    N(\declamb^*)^{1/4}
    &=
    \left(
        \frac{
            2C_2mT^2
        }{
            3C_1
            \max\{H_f^2,G_f^2\}
            \psi\left(
                \norm{\enc}_{\hild{2}{d}}^2
            \right)
        }
    \right)^{1/5}
    q^{2/5}N^{3/5}
    \nonumber\\
    &\longrightarrow
    \infty.
\end{align}
Thus, there exists $N_0$ such that, for every $N\geq N_0$, we have
\(
    0<\declamb^*\leq1,
    \text{ and }
    N(\declamb^*)^{1/4}\geq c_0.
\)
Finally,
\begin{align}
    \declamb^*
    \asymp
    \left(
        \frac{q}{N}
    \right)^{8/5},
\end{align}
as claimed.

\subsection{Proof of Theorem~\ref{th:efficient_adv_encoder}}

The proof follows the same argument as that of Theorem~3 in~\cite{moradi2026learning}. Indeed, for a fixed decoder smoothing parameter $\declamb$, the encoder-dependent component of the loss upper bound has the exact same form as in the straggler setting. The only additional term is the adversarial perturbation $\Delta_{q,N}$, which is independent of the encoder and therefore does not affect the optimization or the comparison argument used in the original proof.

\subsection{Proof of Corollary~\ref{cor:total_convergence_rate}}
\label{sec:proof_total_convergence_rate}

\begin{proof}
Let $s_{\mathbf X}\in\hild{2}{d}$ be the natural cubic spline
interpolating the input batch in encoder design points:
\begin{align}
    s_{\mathbf X}(\alpha_k)
    =
    \mathbf{x}_k,
    \qquad k\in[K].
\end{align}
Because the input batch and encoder design points are fixed,
\(
    \norm{s_{\mathbf X}}_{\hild{2}{d}}
\)
is independent of $N$. Moreover,
\begin{align}\label{eq:natural_cub_enc}
    \frac{1}{K}
    \sum_{k=1}^K
    \norm{
        s_{\mathbf X}(\alpha_k)-\mathbf{x}_k
    }_2^2
    =
    0.
\end{align}

By Theorem~\ref{th:efficient_adv_encoder}, we have
\begin{align}\label{eq:efficient_adv_encoder_loss_bound}
    \objadv
    &\leq
    \Delta_{q,N}
    +
    2G_f^2
    \bigg[
        \frac{1}{K}
        \sum_{k=1}^K
        \norm{
            \encdagger(\alpha_k)
            -
            \mathbf{x}_k
        }_2^2
        +
        \enclamb
        \left(
            R_{\mathrm{e}}
            +
            \norm{
                (\encdagger)''
            }_{\lpd{2}{d}}^2
        \right)
    \bigg].
\end{align}
Since $\encdagger(\cdot)$ minimizes the regularized objective
\[
    \frac{1}{K}
    \sum_{k=1}^K
    \norm{
        u(\alpha_k)
        -
        \mathbf{x}_k
    }_2^2
    +
    \enclamb
    \norm{u''}_{\lpd{2}{d}}^2
\]
over all $u \in \hild{2}{d}$, its objective value is upper bounded by that of the natural cubic spline interpolant $s_{\mathbf{X}} \in \hild{2}{d}$. Consequently, we obtain
\begin{align}
    \objadv
    &\leq
    \Delta_{q,N}
    +
    2G_f^2
    \bigg[
        \frac{1}{K}
        \sum_{k=1}^K
        \norm{
            \encdagger(\alpha_k)
            -
            \mathbf{x}_k
        }_2^2
        +
        \enclamb
        \left(
            R_{\mathrm{e}}
            +
            \norm{
                (\encdagger)''
            }_{\lpd{2}{d}}^2
        \right)
    \bigg]
    \nonumber \\ 
    &\leq
    \Delta_{q,N}
    +
    2G_f^2
    \bigg[
        \frac{1}{K}
        \sum_{k=1}^K
        \norm{
            s_{\mathbf{X}}(\alpha_k)
            -
            \mathbf{x}_k
        }_2^2
        +
        \enclamb
        \left(
            R_{\mathrm{e}}
            +
            \norm{
                (s_{\mathbf{X}})''
            }_{\lpd{2}{d}}^2
        \right)
    \bigg]
    \nonumber \\ 
    &\lec{(a)}{=}
    \Delta_{q,N}
    +
    2G_f^2
    \enclamb
    \left(
        R_{\mathrm{e}}
        +
        \norm{
            (s_{\mathbf{X}})''
        }_{\lpd{2}{d}}^2
    \right),
\end{align}
where $\text{(a)}$ follows from the exact interpolation property $s_{\mathbf{X}}(\alpha_k) = \mathbf{x}_k$ in~\eqref{eq:natural_cub_enc}. Substituting the definitions of $\Delta_{q,N}$ and $\enclamb$ from~\eqref{eq:encoder_independent_adv_term} and~\eqref{eq:adv_encoder_regularization}, respectively, yields
\begin{align}\label{eq:proof_total_preoptimization}
    \objadv
    \leq
    A_0
    \left(
        \declamb^{3/4}
        +
        N^{-3}
    \right)
    +
    B_0
    \frac{q^2}{N^2}
    \declamb^{-1/2},
\end{align}
where $A_0 := 2G_f^2\big(R_{\mathrm{e}} + \norm{(s_{\mathbf{X}})''}_{\lpd{2}{d}}^2\big)(m_1 + m_2 R_{\mathrm{e}})$ and $B_0 := C_2 m T^2$.

For $q\geq1$, choose
\(
    \declamb
    \asymp
    \left(
        \frac{q}{N}
    \right)^{8/5}
\)
. Then we have
\(
    \declamb^{3/4}
    \asymp
    \left(
        \frac{q}{N}
    \right)^{6/5}
    \text{ and }
    \frac{q^2}{N^2}
    \declamb^{-1/2}
    \asymp
    \left(
        \frac{q}{N}
    \right)^{6/5}
\)
. Consequently,
\begin{align}\label{eq:proof_total_rate_intermediate}
    \objadv
    =
    \mathcal{O}
    \left(
        \left(
            \frac{q}{N}
        \right)^{6/5}
        +
        N^{-3}
    \right).
\end{align}
Since $q\leq C_qN^\theta$, we obtain
\begin{align}
    \left(
        \frac{q}{N}
    \right)^{6/5}
    \leq
    C_q^{6/5}
    N^{\frac{6}{5}(\theta-1)}.
\end{align}
For $\theta\in[0,1)$,
\(
    \frac{6}{5}(\theta-1)
    \geq
    -\frac{6}{5}
    >
    -3.
\)
Thus, the $N^{-3}$ term is asymptotically no larger than the first
term, and
\begin{align}
    \objadv
    =
    \mathcal{O}
    \left(
        N^{\frac{6}{5}(\theta-1)}
    \right).
\end{align}

It remains to verify the conditions of
Theorem~\ref{th:decoder_error_bound}. We have
\begin{align}
    N\declamb^{1/4}
    &\asymp
    N
    \left(
        \frac{q}{N}
    \right)^{2/5}
    =
    N^{3/5}q^{2/5}
    \longrightarrow
    \infty.
\end{align}
Therefore for sufficiently large $N$ it satisfies $N\declamb^{1/4}\geq c_0$. Also, $q=o(N)$ implies $\declamb\to0$. 

When $q=0$, the adversarial term vanishes. Choosing
$\declamb\asymp N^{-4}$ gives an error of order $N^{-3}$, which
also satisfies the stated upper bound.
\end{proof}

\subsection{Proof of Theorem~\ref{th:spline_impos}}
\label{sec:proof_spline_impos}

Fix $f \in \dmsmfunc$, an encoder $\enc \in \hild{2}{d}$, and a decoder smoothing parameter $\declamb > 0$ chosen prior to observing the worker outputs. As established in Section~\ref{sec:smooth_spline} and~\eqref{eq:smoothing_kernel}, the smoothing-spline decoder is a linear operator with respect to the received worker outputs. Specifically, under any admissible adversarial attack $A$, the decoded function takes the form
\begin{align}
    \decstar(z)
    =
    \frac{1}{N}
    \sum_{n=1}^N
    G_{N,\declamb}(z, \beta_n)\overline{\mathbf{y}}_n^{A},
    \qquad z \in \Omega,
\end{align}
where the weight function $G_{N,\declamb}(z, \beta_n)$ depends solely on the decoder design points $\{\beta_n\}_{n=1}^N$, the smoothing parameter $\declamb$, and $N$. Evaluating $\decstar(\cdot)$ at the first target point $\alpha_1$ yields
\begin{align}\label{eq:converse_spline_weights}
    \decstar(\alpha_1)
    =
    \sum_{n=1}^N w_n\overline{\mathbf{y}}_n^{A},
\end{align}
where $w_n := \frac{1}{N}G_{N,\declamb}(\alpha_1, \beta_n)$. Note that the same scalar weights $\{w_n\}_{n=1}^N$ apply across all $m$ output coordinates because the smoothing-spline objective in~\eqref{eq:decoder_opt_adv} decouples coordinatewise.

To determine the sum of these weights, suppose that $\overline{\mathbf{y}}_n^{A} = \mathbf{c}$ for all $n \in [N]$, where $\mathbf{c} \in \mathbb{R}^m$ is an arbitrary non-zero constant vector. In this case, the constant function $\dec(z) \equiv \mathbf{c}$ achieves an objective value of zero in~\eqref{eq:decoder_opt_adv} and is therefore the unique minimizer $\decstar$. Substituting $\overline{\mathbf{y}}_n^{A} = \mathbf{c}$ and $\decstar(\alpha_1) = \mathbf{c}$ into~\eqref{eq:converse_spline_weights} implies that the weights reproduce constants:
\begin{align}\label{eq:converse_weights_sum}
    \sum_{n=1}^N w_n = 1.
\end{align}

Now let $\mathcal{B} \subset [N]$ denote the subset of $q$ worker indices corresponding to the $q$ largest values of $\abs{w_n}$ (breaking ties arbitrarily). Since the average of the $q$ largest magnitudes is at least the overall average across all $N$ indices, it follows from~\eqref{eq:converse_weights_sum} and the triangle inequality that
\begin{align}\label{eq:converse_top_weights}
    \sum_{n\in\mathcal{B}} \abs{w_n}
    \geq
    \frac{q}{N}\sum_{n=1}^N \abs{w_n}
    \geq
    \frac{q}{N}\abs{\sum_{n=1}^N w_n}
    =
    \frac{q}{N}.
\end{align}

For each sign $\sigma \in \{-1, 1\}$, consider the adversarial attack $A_\sigma$ that corrupts the workers in $\mathcal{B}$ by setting
\begin{align}
    \overline{\mathbf{y}}_n^{A_\sigma}
    =
    \sigma T_a \operatorname{sign}(w_n)\mathbf{e}_1,
    \qquad n \in \mathcal{B},
\end{align}
while leaving all remaining worker outputs ($n \notin \mathcal{B}$) uncorrupted, where $\mathbf{e}_1 \in \mathbb{R}^m$ denotes the first standard basis vector and $\operatorname{sign}(0) := 0$. Both $A_1$ and $A_{-1}$ are admissible under Assumption~\textnormal{(A3)}, as each modifies at most $\abs{\mathcal{B}} = q$ worker outputs and satisfies $\overline{\mathbf{y}}_n^{A_\sigma} \in [-T_a, T_a]^m$.

Next, define the uncorrupted residual $b$ and the adversarial perturbation magnitude $a$ in the first output coordinate at $\alpha_1$ by
\begin{align}
    b
    &:=
    \sum_{n\notin\mathcal{B}}
    w_n f_1(\enc(\beta_n))
    -
    f_1(\mathbf{x}_1),\\
    a
    &:=
    T_a\sum_{n\in\mathcal{B}}\abs{w_n}.
\end{align}
By~\eqref{eq:converse_spline_weights}, under attack $A_\sigma$, the estimation error in the first coordinate at $\alpha_1$ is given by $\big[\decstar(\alpha_1) - f(\mathbf{x}_1)\big]_1 = b + \sigma a$. Using the identity
\begin{align}
    \max_{\sigma \in \{-1, 1\}}
    (b + \sigma a)^2
    \geq
    \frac{(b+a)^2 + (b-a)^2}{2}
    =
    b^2 + a^2
    \geq
    a^2,
\end{align}
 together with~\eqref{eq:converse_top_weights}, we lower-bound the worst-case squared Euclidean error at $\alpha_1$ as
\begin{align}
    \max_{\sigma \in \{-1, 1\}}
    \norm{
        \decstar(\alpha_1)
        -
        f(\mathbf{x}_1)
    }_2^2
    \geq
    \max_{\sigma \in \{-1, 1\}}
    (b + \sigma a)^2
    \geq
    a^2
    \geq
    T_a^2\left(\frac{q}{N}\right)^2.
\end{align}
Because the squared errors at the remaining $K - 1$ target points $\{\alpha_k\}_{k=2}^K$ are non-negative, the worst-case risk over all admissible attacks $A$ satisfies
\begin{align}
    \Radv(\enc, \decstar; f)
    &\geq
    \frac{1}{K}
    \max_{\sigma \in \{-1, 1\}}
    \norm{
        \decstar(\alpha_1)
        -
        f(\mathbf{x}_1)
    }_2^2
    \geq
    \frac{T_a^2}{K}
    \left(\frac{q}{N}\right)^2.
\end{align}
Since this lower bound holds for any $\enc \in \hild{2}{d}$ and any $\declamb > 0$, taking the infimum over both parameters yields~\eqref{eq:spline_converse_finite} and completes the proof.

\subsection{Proof of Corollary~\ref{cor:spline_linear_adversarial_fraction}}
\label{sec:proof_spline_linear_adversarial_fraction}

Since $\mu \in (0, 1)$, for all $N \geq \lceil 1/\mu \rceil$, the adversarial budget $q = \lfloor \mu N \rfloor$ satisfies $1 \leq q < N$. Applying Theorem~\ref{th:spline_impos} for every such $N$ yields
\begin{align}
    \inf_{\substack{
        \enc\in\hild{2}{d},\\
        \declamb>0
    }}
    \Radv(\enc, \decstar; f)
    \geq
    \frac{T_a^2}{K}
    \left(
        \frac{\lfloor\mu N\rfloor}{N}
    \right)^2.
\end{align}
Because $\lim_{N\to\infty} \lfloor\mu N\rfloor / N = \mu$, taking the limit inferior as $N \to \infty$ on both sides establishes the corollary.
\section{Experimental Results}\label{sec:experiments}

In this section, we empirically evaluate the proposed adversarial
\gcc{} framework. We consider two types of computing functions: (1) a scalar nonlinear function and (2) a
high-dimensional neural-network inference task. The primary purpose
of these experiments is to  examine the effectiveness of \gcc{} to mitigate the adversarial effect and  convergence-rate prediction
of Corollary~\ref{cor:total_convergence_rate} under a concrete
adversarial strategy. 


\subsection{Experimental Setup}

\textbf{Target Tasks:}
We evaluate our framework on two representative computing tasks:
\begin{itemize}
    \item \textit{Scalar nonlinear function:}
    The first target function is given by
    \begin{align}\label{eq:scalar_experiment_function}
        f_{\mathrm{sc}}(x)
        :=
        x\sin(x),
    \end{align}
    evaluated over the compact interval $[-12, 12]$ to ensure bounded inputs and outputs.

    \item \textit{Neural-network inference:}
    The second task evaluates a LeNet-5 convolutional neural network~\cite{lecun1998gradient} trained on the MNIST dataset~\cite{lecun2010mnist}. Each input is preprocessed into a $32\times32$ grayscale image ($1024$-dimensional vector), and the network outputs a $10$-dimensional vector representing the class probability distribution across digits $0$ through $9$:
    \begin{align}\label{eq:lenet_experiment_function}
        f_{\mathrm{NN}}:
        \mathbb{R}^{1024}
        \to
        \mathbb{R}^{10}.
    \end{align}
\end{itemize}

In both settings, we fix the batch size to $K = 10$. The encoder design points $\{\alpha_k\}_{k=1}^K$ and worker evaluation points $\{\beta_j\}_{j=1}^N$ are placed uniformly over $[0, 1]$ according to~\eqref{eq:equispaced_decoder_points}:
\begin{align}
    \alpha_k
    &=
    \frac{k-1}{K-1},
    \qquad k \in [K],
    \\
    \beta_j
    &=
    \frac{j-1}{N-1},
    \qquad j \in [N].
\end{align}

To evaluate performance across different adversarial regimes, we scale the number of adversarial workers as
\begin{align}\label{eq:experimental_adversarial_scaling}
    q
    :=
    \left\lfloor
        N^\theta
    \right\rfloor,
    \qquad
    0 \leq \theta < 1.
\end{align}

\textbf{Evaluation Metric:}
We evaluate performance via the expected end-to-end mean squared error (MSE):
\begin{align}\label{eq:emp_metric}
    \mathrm{MSE}
    :=
    \mathbb{E}_{\mathbf{x} \sim \mathcal{X}} \left[
        \frac{1}{K}
        \sum_{k=1}^K
        \norm{
            \fhat(\mathbf{x}_k)
            -
            f(\mathbf{x}_k)
        }_2^2
    \right],
\end{align}
where $\mathcal{X}$ denotes the input data distribution. In practice, we approximate \eqref{eq:emp_metric} by computing the empirical average over $10$ independent experimental runs.

\begin{figure*}[t]
    \centering
    \begin{subfigure}[b]{0.48\textwidth}
        \centering
        \includegraphics[
            width=0.95\textwidth
        ]{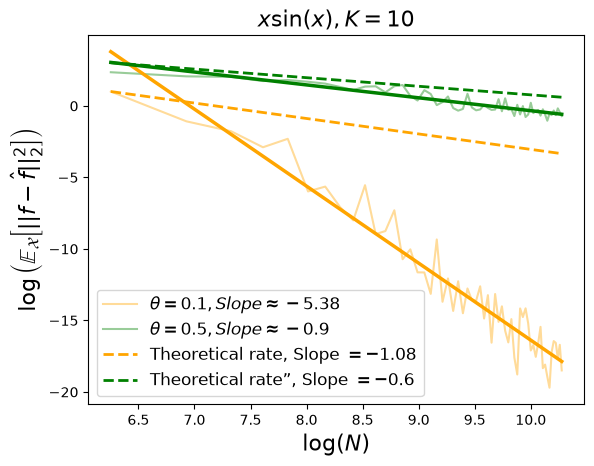}
        \caption{Scalar computation
        $f_{\mathrm{sc}}(x)=x\sin(x)$.}
        \label{fig:rate_xsinx}
    \end{subfigure}
    \hfill
    \begin{subfigure}[b]{0.48\textwidth}
        \centering
        \includegraphics[
            width=0.95\textwidth
        ]{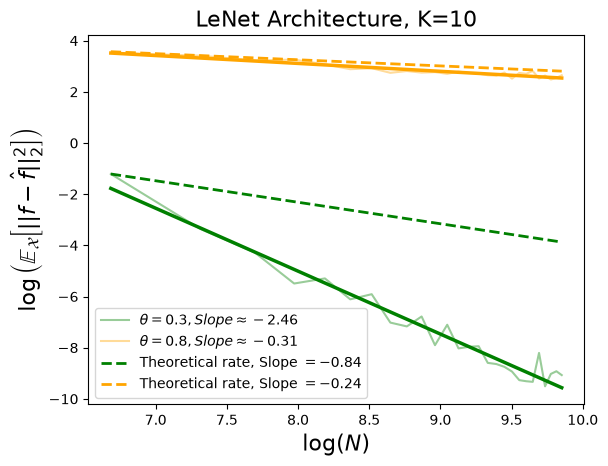}
        \caption{LeNet-5 inference.}
        \label{fig:rate_lenet}
    \end{subfigure}
    \caption{
        Approximation error as a function of the number of workers
        for adversarial budgets
        $q=\lfloor N^\theta\rfloor$.
        The light solid curves show the empirical errors, the dark
        solid lines show the fitted empirical slopes, and the dashed
        lines show the convergence-rate predictions from
        Corollary~\ref{cor:total_convergence_rate}.
    }
    \label{fig:rate_all}
\end{figure*}

\subsection{Adversarial Strategy}

For each encoder design point $\alpha_k$, the adversary selects worker
nodes whose decoder design points $\beta_j$ are closest to
$\alpha_k$. The total adversarial budget $q$ is distributed as evenly
as possible among the $K$ encoder design points. Consequently, approximately $q/K$
outputs near each $\alpha_k$ are corrupted.

The corrupted values are selected differently for the two computing
tasks.

\begin{itemize}
    \item \textit{Scalar nonlinear function:}
    For every selected worker $j$, the adversary replaces the honest
    output $f_{\mathrm{sc}}(\enc(\beta_j))$ by the largest admissible
    scalar value \(
        \overline{y}_j^A
        =
        T_a.
\) We set $T_a=10$ in this experiment.

    \item \textit{LeNet-5 inference:}
    Let
    \begin{align}
        c_j
        :=
        \operatorname*{arg\,max}_{i\in[m]}
        \left[
            f_{\mathrm{NN}}
            \bigl(
                \enc(\beta_j)
            \bigr)
        \right]_i
    \end{align}
    denote the class predicted by the honest worker output. The
    adversary selects a random different class $c_j^A$.
    It then replaces the honest output vector by
    \(
        \overline{\mathbf{y}}_j^A
        =
        \mathbf{e}_{c_j^{A}},
    \)
    where $\mathbf{e}_{c_j^{A}}\in\mathbb{R}^m$ is the corresponding
    standard basis vector. Thus, each corrupted worker outputs a
    one-hot vector concentrated on a class different from the class
    predicted by its honest output.
\end{itemize}

Both attacks concentrate the corrupted outputs near the locations at
which the final estimates are evaluated to inject maximum noise.  These attacks represent
particular admissible strategies and do not restrict the general
adversarial model considered in the theoretical analysis.

\subsection{Convergence-Rate Evaluation}

Under the scaling in
\eqref{eq:experimental_adversarial_scaling},
Corollary~\ref{cor:total_convergence_rate} gives
\(
    \objadv
    =
    \mathcal{O}(
        N^{-\frac{6}{5}(1-\theta)}
    ).
\)
Therefore, the predicted slope in a log-log plot of the approximation
error versus $N$ is
\(
    -\frac{6}{5}(1-\theta).
\)
Figure~\ref{fig:rate_all} compares these predicted slopes with the
slopes fitted to the empirical curves. 


In every considered regime, the empirical slope is more negative than
the corresponding theoretically guaranteed slope. Thus, over the
investigated range of worker populations, the approximation error
decays at least as rapidly as predicted by
Corollary~\ref{cor:total_convergence_rate}. As expected, increasing
$\theta$ slows the convergence because the adversarial population grows more rapidly with $N$.


\section{Conclusion and Future Directions}\label{sec:conclusion}

In this paper, we introduced an adversarially robust framework for general coded computing (\gcc{}) that mitigates the impact of malicious worker nodes through smoothness regularization rather than explicit error localization. By distributing computations via a smooth encoder and aggregating all returned outputs with a regularized decoder, we established a non-asymptotic worst-case error bound. This theoretical analysis proves that our proposed spline-based scheme achieves linear master-side complexity and guarantees a vanishing approximation error whenever the adversarial population grows sublinearly with the total number of workers. 

Several open questions and directions remain for future investigation. First, our current design assumes the adversarial budget is known \textit{a priori} in order to tune the optimal smoothing parameters. Developing adaptive procedures capable of estimating the effective corruption level directly from the received outputs would significantly enhance the framework's practical applicability. Second, future work should address the simultaneous presence of stragglers and adversarial workers, as missing observations and corrupted outputs distort the regularized decoder via fundamentally different mechanisms. Finally, establishing a matching minimax lower bound for the $\mathcal{O}((q/N)^{6/5})$ convergence rate would clarify whether our achievable exponent is strictly optimal, or if exploring higher-order regularization penalties and alternative reproducing kernels could yield sharper convergence rates.

\section*{Acknowledgment}
The work has been partially supported by the National Science Foundation under Grant CCF-2348638.

\bibliographystyle{ieeetr}

\section{Biographies}
\begin{IEEEbiographynophoto}{Parsa Moradi}
is a Ph.D. candidate in the Department of Electrical and Computer Engineering at the University of Minnesota Twin Cities. He received his B.Sc. and M.Sc. degrees from the Department of Electrical Engineering, Sharif University of Technology, Tehran, Iran, in 2016 and 2019, respectively. His research interests include deep learning, generative AI,  distributed machine learning, and AI for healthcare.
\end{IEEEbiographynophoto}

\begin{IEEEbiographynophoto}{Hanzaleh Akbarinodehi}
is a Ph.D. candidate in the Department of Electrical and Computer Engineering at the University of Minnesota Twin Cities. He received his B.Sc. and M.Sc. degrees from the Department of Electrical Engineering, Sharif University of Technology, Tehran, Iran, in 2016 and 2018, respectively. His research interests include information theory and coding, blockchain technology, game theory, and decentralized systems. 
\end{IEEEbiographynophoto}

\begin{IEEEbiographynophoto}{Mohammad Ali Maddah-Ali}
(IEEE Fellow, 2023) is an Associate Professor at the University of Minnesota Twin Cities. He received his B.Sc. degree in Electrical Engineering from Isfahan University of Technology, his M.A.Sc. degree from the University of Tehran, and his Ph.D. in Electrical and Computer Engineering from the University of Waterloo, Canada, in 2007.

From 2007 to 2008, he was with the Wireless Technology Laboratories at Nortel Networks, Ottawa, ON, Canada. He then held a Postdoctoral Fellowship at the Department of Electrical Engineering and Computer Sciences, University of California, Berkeley, from 2008 to 2010. From September 2010 to September 2020, he served as a Communication Research Scientist at Nokia Bell Labs, NJ, USA. 

Dr. Maddah-Ali is the recipient of several honors, including the NSERC Postdoctoral Fellowship (2007), the Best Paper Award at the IEEE International Conference on Communications (ICC) in 2014, the IEEE Communications Society and IEEE Information Theory Society Joint Paper Award in 2015, and the IEEE Information Theory Society Paper Award in 2016. He served as an Associate Editor for the IEEE Transactions on Information Theory (2019–2022) and as Lead Editor for the IEEE Journal on Selected Areas in Information Theory. He is currently a distinguished lecturer of the IEEE Information Theory Society.
\end{IEEEbiographynophoto}

\newpage
\appendices
\section{Additional Functional-Analytic Details}\label{app:functional_details}
\begin{definition}[Weak Derivatives and Vector-Valued Norms]
\label{def:weak_deriv}
A locally integrable scalar function $v$ has a weak derivative $w$ on $\Omega \subset \mathbb{R}$ if
\begin{align}
    \int_\Omega
    v(t)\phi'(t)
    \,dt
    =
    -
    \int_\Omega
    w(t)\phi(t)
    \,dt
\end{align}
for every infinitely differentiable test function $\phi$ with compact support in $\Omega$. Weak derivatives of vector-valued functions are defined coordinate-wise.
\end{definition}

\begin{definition}[$L^p$-norm]
\label{def:lp_norm}
For a vector-valued function $u=[u_1,\ldots,u_M]^T$, the $L^p$-norm on a domain $\Omega \subset \mathbb{R}$ is defined as follows:
\begin{align}
    \norm{u}_{L^p(\Omega;\mathbb{R}^M)}
    :=
    \left\{
        \int_\Omega
        \norm{u(t)}_2^p
        \,dt
    \right\}^{1/p},
\end{align}
for $1\leq p<\infty$, and 
\begin{align}
    \norm{u}_{L^\infty(\Omega;\mathbb{R}^M)}
    :=
   \operatorname*{ess\,sup}_{z\in\Omega} \norm{u_i(z)}_2,
\end{align}
when $p=\infty$.
\end{definition}

\begin{definition}[Local Sobolev spaces]
A function $g:\Omega\to\mathbb{R}^M$ belongs to local Sobolev space $\soblocm{r}{p}$ if
$g\in \soblocm{r}{p}$ for every compact set $V\subset \Omega$.
\end{definition}

\begin{lemma}[Theorem 8 \cite{moradi2026learning}]\label{th:interp_ineq2}
$\sobm{2}{2} \subseteq \sobm{1}{1} \subseteq \soblocm{1}{1}$.
\end{lemma}

\begin{definition}[Sobolev Space with Compact Support]\label{def:sob_compact}
Given bounded interval $\Omega = [a, b]$, Sobolev order $r \in \mathbb{N}$, and exponent $p \in [1,\infty)$, we define $W_0^{r,p}(\Omega; \mathbb{R}^M)$ as the subspace of functions $g \in W^{r,p}(\Omega; \mathbb{R}^M)$ satisfying vanishing  boundary conditions
\begin{equation}
g^{(j)}(a) = 0, \qquad \text{for all } j = 0, 1, \dots, m-1.
\end{equation}
This space is usually equipped with the norm based on the highest-order derivative:
\begin{equation}
\|g\|_{W_0^{r,p}(\Omega; \mathbb{R}^M)} := \|g^{(m)}\|_{L^p(\Omega; \mathbb{R}^M)}, \qquad 1 \le p < \infty,
\end{equation}
and for $p = \infty$,
\begin{equation}
\|g\|_{W_0^{r,\infty}(\Omega; \mathbb{R}^M)} := \|g^{(m)}\|_{L^\infty(\Omega; \mathbb{R}^M)}.
\end{equation}
When $p = 2$, this space is an RKHS, and we denote it by $\hilmz{m}$.
\end{definition}

\begin{remark}[Explicit kernel form]\label{rem:kernel}
The scalar reproducing kernel of $\hilz{2}$ for $\Omega=[0,1]$ is
\begin{align}
    \phi_0(s,t)
    = \int_0^1 (s-u)_+ (t-u)_+ \, du
    = \frac{\min(s,t)^2\,\bigl(3\max(s,t)-\min(s,t)\bigr)}{6},
\end{align}
and the reproducing kernel of $\hil{2}$ under the norm above is
$\phi(s,t) = 1 + st + \phi_0(s,t)$.
\end{remark}

\section{Proof of Lemmas}\label{app:lemma_proof}

\subsection{Proof of Lemma~\ref{lem:clean_spline_approximation}}
\label{proof_lem:clean_spline_approximation}
Define the smoothing-spline approximation error
\begin{align}\label{eq:spline_approximation_error}
    e(z)
    :=
    \smop{n}{\lambda}{h}(z)-h(z).
\end{align}
Since $\smop{n}{\lambda}{h},h\in\hild{2}{m}$, it follows immediately that
\begin{align}
    e\in\hild{2}{m}.
\end{align}

We first bound the supremum norm of $e(\cdot)$ using its $\lpd{2}{m}$ norm
and the $\lpd{2}{m}$ norm of its first derivative. 

Using the smoothing-spline operator defined in \eqref{eq:smoothing_spline}, we have
\begin{align}
    \smop{n}{\lambda}{h}
    =
    \operatorname*{arg\,min}_{u\in\hild{2}{m}}
    \left\{
        \frac{1}{n}
        \sum_{i=1}^n
        \norm{
            u(t_i)-h(t_i)
        }_2^2
        +
        \lambda
        \norm{u''}_{\lpd{2}{m}}^2
    \right\}.
\end{align}
For simplicity, let $$s:=\smop{n}{\lambda}{h}.$$ 
We first derive a useful orthogonality property of the
smoothing-spline error. For an arbitrary direction
$v\in\hild{2}{m}$, define the one-dimensional function
\begin{align}
    \Gamma_v(\varepsilon)
    &:=
    \frac{1}{n}
    \sum_{i=1}^n
    \norm{
        s(t_i)
        +
        \varepsilon v(t_i)
        -
        h(t_i)
    }_2^2
    +
    \lambda
    \norm{
        s''
        +
        \varepsilon v''
    }_{\lpd{2}{m}}^2,
    \qquad
    \varepsilon\in\mathbb{R}.
\end{align}
Because $s(\cdot)$ minimizes the smoothing-spline objective,
$\varepsilon=0$ minimizes $\Gamma_v(\varepsilon)$. Therefore,
\begin{align}
    \Gamma_v'(0)
    =
    0.
\end{align}
Differentiating the two quadratic terms gives
\begin{align}
    \Gamma_v'(0)
    &=
    \frac{2}{n}
    \sum_{i=1}^n
    \left\langle
        s(t_i)-h(t_i),
        v(t_i)
    \right\rangle
    +
    2\lambda
    \left\langle
        s'',v''
    \right\rangle_{\lpd{2}{m}},
\end{align}
where
\begin{align}
    \left\langle
        s'',v''
    \right\rangle_{\lpd{2}{m}}
    :=
    \int_0^1
    \left\langle
        s''(t),v''(t)
    \right\rangle
    dt.
\end{align}
Recalling that $e=s-h$ and dividing by two, we obtain the
first-order optimality condition
\begin{align}\label{eq:spline_error_first_order}
    \frac{1}{n}
    \sum_{i=1}^n
    \left\langle
        e(t_i),v(t_i)
    \right\rangle
    +
    \lambda
    \left\langle
        s'',v''
    \right\rangle_{\lpd{2}{m}}
    =
    0
\end{align}
for every $v\in\hild{2}{m}$. This is the standard variational
optimality condition for smoothing splines
\cite{wahba1975smoothing,wahba1990spline}.

Now fix an arbitrary vector $\mathbf{v}_0\in\mathbb{R}^m$ and choose
the constant perturbation
\begin{align}
    v(t)
    \equiv
    \mathbf{v}_0.
\end{align}
Since $v''(t)=\mathbf{0}$, condition
\eqref{eq:spline_error_first_order} becomes
\begin{align}
    0
    &=
    \frac{1}{n}
    \sum_{i=1}^n
    \left\langle
        e(t_i),\mathbf{\mathbf{v}_0}
    \right\rangle
    =
    \left\langle
        \frac{1}{n}
        \sum_{i=1}^n
        e(t_i),
        \mathbf{\mathbf{v}_0}
    \right\rangle.
\end{align}
Because this equality holds for every
$\mathbf{v}_0\in\mathbb{R}^m$, we must have
\begin{align}\label{eq:spline_error_zero_sum}
    \sum_{i=1}^n e(t_i)
    =
    \mathbf{0}.
\end{align}
Equivalently, for every coordinate $\ell\in[m]$,
\begin{align}\label{eq:coordinate_error_zero_sum}
    \sum_{i=1}^n e_\ell(t_i)
    =
    0.
\end{align}

We next show that every coordinate error $e_\ell$ vanishes somewhere
on $[0,1]$. If $e_\ell(t_i)=0$ for some $i$, we simply take
$z_\ell=t_i$. Otherwise, all the values
$\{e_\ell(t_i)\}_{i=1}^n$ are nonzero. In this case,
\eqref{eq:coordinate_error_zero_sum} implies that they cannot all
have the same sign. Hence, there exist two consecutive design points
$t_i$ and $t_{i+1}$ such that
\begin{align}
    e_\ell(t_i)e_\ell(t_{i+1})
    <
    0.
\end{align}
Since $e_\ell\in H^2(0,1)\subset H^1(0,1)$, it has a continuous,
absolutely continuous representative on $[0,1]$. The intermediate value theorem therefore
ensures the existence of a point
\begin{align}
    z_\ell
    \in
    [t_i,t_{i+1}]
\end{align}
such that
\begin{align}\label{eq:coordinate_error_root}
    e_\ell(z_\ell)
    =
    0.
\end{align}

Because $e_\ell\in H^1(0,1)$, the function $e_\ell^2$ is absolutely
continuous and satisfies
\begin{align}
    \frac{d}{dt}
    e_\ell(t)^2
    =
    2e_\ell(t)e_\ell'(t)
\end{align}
for almost every $t\in(0,1)$. Hence, for every $z\in[0,1]$,
\begin{align}
    |e_\ell(z)|^2
    &=
    \left|
        e_\ell(z)^2
        -
        e_\ell(z_\ell)^2
    \right|
    \nonumber\\
    &=
    2
    \left|
        \int_{z_\ell}^{z}
        e_\ell(t)e_\ell'(t)\,dt
    \right|
    \nonumber\\
    &\leq
    2
    \int_0^1
    |e_\ell(t)|
    |e_\ell'(t)|
    \,dt
    \nonumber\\
    &\leq
    2
    \norm{e_\ell}_{\lp{2}}
    \norm{e_\ell'}_{\lp{2}},
\label{eq:coordinate_error_supremum}
\end{align}
where the last step follows from the Cauchy--Schwarz inequality.

Taking the supremum over $z$ in
\eqref{eq:coordinate_error_supremum}, summing over
$\ell\in[m]$, and applying the Cauchy--Schwarz inequality to the
coordinate index give
\begin{align}
    \norm{e}_{\lpd{\infty}{m}}^2
    &\leq
    \sum_{\ell=1}^m
    \norm{e_\ell}_{\lp{\infty}}^2
    \nonumber\\
    &\leq
    2
    \sum_{\ell=1}^m
    \norm{e_\ell}_{\lp{2}}
    \norm{e_\ell'}_{\lp{2}}
    \nonumber\\
    &\leq
    2
    \norm{e}_{\lpd{2}{m}}
    \norm{e'}_{\lpd{2}{m}}.
\label{eq:spline_error_supnorm_bound}
\end{align}
Taking square roots yields the slightly sharper estimate
\begin{align}
    \norm{e}_{\lpd{\infty}{m}}
    \leq
    \sqrt{2}
    \norm{e}_{\lpd{2}{m}}^{1/2}
    \norm{e'}_{\lpd{2}{m}}^{1/2}.
\end{align}
In particular, since $\sqrt{2}\leq2$,
\begin{align}
    \norm{e}_{\lpd{\infty}{m}}
    \leq
    2
    \norm{e}_{\lpd{2}{m}}^{1/2}
    \norm{e'}_{\lpd{2}{m}}^{1/2}.
\label{eq:spline_error_interpolation}
\end{align}

Next, we bound the two quantities on the right-hand side of \eqref{eq:spline_error_interpolation}. Recall from Definition~\ref{def:spacing} that for any ordered design set $\mathcal{T}=\{t_1,\ldots,t_n\}$ where $t_1 < \ldots < t_n$, the terms $\delmax{\tau}$ and $\delmin{\tau}$ denote the maximum and minimum consecutive spacings in $\mathcal{T}$, respectively.

Authors in ~\cite{ragozin1983error} have shown that, for a
scalar target function $g\in H^2(\Omega;\mathbb{R})$ and its
smoothing-spline approximation $\smop{n}{\lambda}{g}$ over $\mathcal{T}$,
there exist constants $G_0,G_1>0$ such that
\begin{align}\label{eq:proof_scalar_spline_error_j0}
    \norm{g-\smop{n}{\lambda}{g}}_{\lp{2}}^2
    \leq
    G_0
    L_{\mathcal{T}}
    \norm{g''}_{\lp{2}}^2
\end{align}
and
\begin{align}\label{eq:proof_scalar_spline_error_j1}
    \norm{(g-\smop{n}{\lambda}{g})'}_{\lp{2}}^2
    \leq
    G_1
    L_{\mathcal{T}}^{1/2}
    \left(
        1+\frac{L_{\mathcal{T}}}{16}
    \right)^{1/2}
    \norm{g''}_{\lp{2}}^2,
\end{align}
where
\begin{align}\label{eq:proof_LT_definition}
    L_{\mathcal{T}}
    &:=
    p_2\left(
        \frac{
            \delmax{\mathcal{T}}
        }{
            \delmin{\mathcal{T}}
        }
    \right)
    \frac{
        n\delmax{\mathcal{T}}
    }{4}
    \lambda
    +
    D
    \delmax{\mathcal{T}}^4.
\end{align}
Here, $p_2$ is a degree-two polynomial with positive coefficients,
and $D>0$ is an absolute constant.

The vector-valued smoothing-spline objective separates across its
$m$ output coordinates. Applying
\eqref{eq:proof_scalar_spline_error_j0} and
\eqref{eq:proof_scalar_spline_error_j1} to each pair
$(h_\ell,s_\ell)$ and summing over $\ell\in[m]$ gives
\begin{align}
    \norm{e}_{\lpd{2}{m}}^2
    &\leq
    G_0
    L_{\mathcal{T}}
    \norm{h''}_{\lpd{2}{m}}^2,
\label{eq:proof_vector_spline_error_j0}\\
    \norm{e'}_{\lpd{2}{m}}^2
    &\leq
    G_1
    L_{\mathcal{T}}^{1/2}
    \left(
        1+\frac{L_{\mathcal{T}}}{16}
    \right)^{1/2}
    \norm{h''}_{\lpd{2}{m}}^2,
\label{eq:proof_vector_spline_error_j1}
\end{align}
where $L_{\mathcal{T}}$ is the design factor in
\eqref{eq:proof_LT_definition} for
\begin{align}
    \mathcal{T}
    =
    \{
        t_1,\ldots,t_n
    \}.
\end{align}

For the equidistant design points
\(
    t_i=(i-1)/(n-1),
\)
we have
\begin{align}
    t_{i+1}-t_i
    =
    \frac{1}{n-1},
    \qquad
    i\in[n-1].
\end{align}
Hence,
\begin{align}
    \delmax{\mathcal{T}}
    =
    \delmin{\mathcal{T}}
    =
    \frac{1}{n-1},
\end{align}
and therefore
\begin{align}
    L_{\mathcal{T}}
    &=
    p_2(1)
    \frac{n}{4(n-1)}
    \lambda
    +
    \frac{D}{(n-1)^4}
    \nonumber\\
    &\lec{(a)}{\leq}
    \frac{p_2(1)}{2}
    \lambda
    +
    16D n^{-4}
    \nonumber\\
    &\leq
    C_L
    \left(
        \lambda+n^{-4}
    \right),
\label{eq:proof_equidistant_Ln}
\end{align}
where $C_L>0$ is independent of $n$ and $\lambda$ and (a) follows from
\(
    \frac{n}{n-1}
    \leq
    2
    \text{ and }
    \frac{1}{(n-1)^4}
    \leq
    16n^{-4},
    \text{ for } n\geq2.
\)

Combining \eqref{eq:spline_error_supnorm_bound},
\eqref{eq:proof_vector_spline_error_j0}, and
\eqref{eq:proof_vector_spline_error_j1}, we obtain
\begin{align}
    \norm{e}_{\lpd{\infty}{m}}^2
    &\leq
    2
    \norm{e}_{\lpd{2}{m}}
    \norm{e'}_{\lpd{2}{m}}
    \nonumber\\
    &\leq
    2
    \left(
        G_0G_1
    \right)^{1/2}
    L_{\mathcal{T}}^{3/4}
    \left(
        1+\frac{L_{\mathcal{T}}}{16}
    \right)^{1/4}
    \norm{h''}_{\lpd{2}{m}}^2.
\label{eq:proof_uniform_spline_intermediate}
\end{align}

Since $0<\lambda\leq1$, it follows from
\eqref{eq:proof_equidistant_Ln} that
\begin{align}
    L_{\mathcal{T}}
    \leq
    2C_L.
\end{align}
Consequently,
\begin{align}
    \left(
        1+\frac{L_{\mathcal{T}}}{16}
    \right)^{1/4}
    \leq
    \left(
        1+\frac{2C_L}{16}
    \right)^{1/4},
\end{align}
which is independent of $n$, $\lambda$, and $h$. Absorbing this
factor and the constants $G_0$, $G_1$, and $C_L$ into a single
constant gives
\begin{align}
    \norm{e}_{\lpd{\infty}{m}}^2
    &\leq
    C_{\mathrm{sp}}
    \left(
        \lambda+n^{-4}
    \right)^{3/4}
    \norm{h''}_{\lpd{2}{m}}^2
    \nonumber\\
    &\leq
    C_{\mathrm{sp}}
    \left(
        \lambda^{3/4}
        +
        n^{-3}
    \right)
    \norm{h''}_{\lpd{2}{m}}^2,
\end{align}
where the last inequality follows from
\(
    (a+b)^{3/4}
    \leq
    a^{3/4}+b^{3/4}
    \text{ for }
    a,b\geq0.
\)

Finally, using $e=\smop{n}{\lambda}{h}-h$, we conclude that
\begin{align}
    \sup_{z\in\Omega}
    \norm{
        \smop{n}{\lambda}{h}(z)-h(z)
    }_2^2
    \leq
    C_{\mathrm{sp}}
    \left(
        \lambda^{3/4}
        +
        n^{-3}
    \right)
    \norm{h''}_{\lpd{2}{m}}^2.
\end{align}
This proves the lemma.

\subsection{Proof of Lemma~\ref{lem:composition_regularity}}
\label{proof_lem:composition_regularity}

To bound the second derivative of the composite function, we make use of the following Gagliardo--Nirenberg-type interpolation inequality on bounded intervals.

\begin{lemma}[{\cite[Theorem~7.34]{leoni2024first}}]
\label{th:interp_ineq}
Let $\Omega \subseteq \mathbb{R}$ be an open interval and let $g \in \soblocm{1}{1}$. Assume $1 \leq p, v, r \leq \infty$ and $r \geq v$. Then, for every $0 < \ell < \Leb(\Omega)$,
\begin{align}\label{eq:leoni_interp}
    \norm{g}_{\lpm{r}}
    \leq
    \ell^{\frac{1}{r}-\frac{1}{v}}\norm{g}_{\lpm{v}}
    +
    \ell^{1 - \frac{1}{p}+\frac{1}{r}}\norm{g'}_{\lpm{p}}.
\end{align}
\end{lemma}

We now proceed with the proof of Lemma~\ref{lem:composition_regularity}. Let $g := f \circ u = [g_1, \ldots, g_m]^\top$, where $u = [u_1, \ldots, u_d]^\top \in \hild{2}{d}$. By the second-order chain rule for Sobolev functions, for each $j \in [m]$ and almost every $z \in \Omega$, we have
\begin{align}
    g_j''(z)
    &=
    u'(z)^\top
    D^2f_j(u(z))
    u'(z)
    +
    \nabla f_j(u(z))^\top
    u''(z).
\end{align}
Applying the triangle inequality, the Cauchy--Schwarz inequality, and the definition of the operator norm, together with the regularity bounds in Assumption~\textnormal{(A2)}, yields
\begin{align}
    \abs{g_j''(z)}
    &\leq
    \norm{D^2f_j(u(z))}_{\op} \norm{u'(z)}_2^2
    +
    \norm{\nabla f_j(u(z))}_2 \norm{u''(z)}_2
    \nonumber\\
    &\leq
    H_j\norm{u'(z)}_2^2
    +
    G_j\norm{u''(z)}_2.
\end{align}
Squaring both sides and using the elementary inequality $(a+b)^2 \leq 2a^2 + 2b^2$, we obtain
\begin{align}
    \abs{g_j''(z)}^2
    \leq
    2H_j^2\norm{u'(z)}_2^4
    +
    2G_j^2\norm{u''(z)}_2^2.
\end{align}
Integrating over $\Omega$ and summing over $j \in [m]$, while recalling that $H_f^2 = \sum_{j=1}^m H_j^2$ and $G_f^2 = \sum_{j=1}^m G_j^2$, gives
\begin{align}\label{eq:composition_pre_embedding}
    \norm{g''}_{\lpd{2}{m}}^2
    =
    \sum_{j=1}^m
    \int_\Omega \abs{g_j''(z)}^2 \, dz
    \leq
    2H_f^2
    \int_\Omega
    \norm{u'(z)}_2^4 \, dz
    +
    2G_f^2
    \norm{u''}_{\lpd{2}{d}}^2.
\end{align}

Next, we bound the quartic integral $\int_\Omega \norm{u'(z)}_2^4 \, dz$ using Lemma~\ref{th:interp_ineq}. First, observe that by Minkowski's inequality in $\lpd{2}{}$,
\begin{align}\label{eq:quartic_minkowski}
    \int_\Omega \norm{u'(z)}_2^4 \, dz
    =
    \int_\Omega
    \left(
        \sum_{i=1}^d \abs{u_i'(z)}^2
    \right)^2 dz
    &=
    \norm{
        \sum_{i=1}^d (u_i')^2
    }_{\lpd{2}{}}^2
    \nonumber \\ &\leq
    \left(
        \sum_{i=1}^d \norm{(u_i')^2}_{\lpd{2}{}}
    \right)^2
    =
    \left(
        \sum_{i=1}^d \norm{u_i'}_{\lpd{4}{}}^2
    \right)^2.
\end{align}
Since $u \in \hild{2}{d}$, each coordinate $u_i$ satisfies 
\begin{align}
    u_i' \in \sobmd{1}{1}{d} \lec{(a)}{\subseteq} \soblocmd{1}{1}{d},
\end{align}
where (a) follows from Lemma~\ref{th:interp_ineq2}.

Fix any $\ell \in (0, \Leb(\Omega))$ and apply Lemma~\ref{th:interp_ineq} to $u_i'$ with $r = 4$ and $p = v = 2$:
\begin{align}
    \norm{u_i'}_{\lpd{4}{}}
    &\leq
    \ell^{-1/4}\norm{u_i'}_{\lpd{2}{}}
    +
    \ell^{3/4}\norm{u_i''}_{\lpd{2}{}}
    \nonumber\\
    &\leq
    C_\ell
    \left(
        \norm{u_i'}_{\lpd{2}{}}
        +
        \norm{u_i''}_{\lpd{2}{}}
    \right),
\end{align}
where $C_\ell := \max\{\ell^{-1/4}, \ell^{3/4}\}$ (for instance, since $\Leb(\Omega) = 1$, choosing $\ell = 1/2$ gives $C_{\ell} = 2^{1/4}$). Squaring both sides and applying $(a+b)^2 \leq 2a^2 + 2b^2$ yields
\begin{align}
    \norm{u_i'}_{\lpd{4}{}}^2
    \leq
    2C_\ell^2
    \left(
        \norm{u_i'}_{\lpd{2}{}}^2
        +
        \norm{u_i''}_{\lpd{2}{}}^2
    \right).
\end{align}
Summing over $i \in [d]$ gives
\begin{align}
    \sum_{i=1}^d \norm{u_i'}_{\lpd{4}{}}^2
    &\leq
    2C_\ell^2
    \left(
        \norm{u'}_{\lpd{2}{d}}^2
        +
        \norm{u''}_{\lpd{2}{d}}^2
    \right)
    \nonumber \\ &\lec{(a)}{\leq}
    2C_\ell^2
    \norm{u}_{\sobmd{2}{2}{d}}^2
    \nonumber \\ &\lec{(b)}{\leq}
    2C_\ell^2\theta^2_\text{sob}
    \norm{u}_{\hild{2}{d}}^2,
\end{align}
 
where (a) follows from $
        \norm{u'}_{\lpd{2}{d}}^2
        +
        \norm{u''}_{\lpd{2}{d}}^2
     \leq \norm{u}_{\sobmd{2}{2}{d}}^2$ and (b) comes from the equivalency of $\norm{\cdot}_{\sobmd{2}{2}{d}}$ and $\norm{\cdot}_\hild{2}{d}$ with a factor $\theta_\text{sob}$ (see Section~\ref{sec:equiv_sob_norm} for more details).

Substituting this bound into~\eqref{eq:quartic_minkowski}, we obtain
\begin{align}\label{eq:quartic_final_bound}
    \int_\Omega \norm{u'(z)}_2^4 \, dz
    \leq
    C_\ell^4\theta^4_\text{sob}
    \norm{u}_{\hild{2}{d}}^4.
\end{align}

Finally, substituting~\eqref{eq:quartic_final_bound} and the bound $\norm{u''}_{\lpd{2}{d}}^2 \leq \norm{u}_{\hild{2}{d}}^2$ into~\eqref{eq:composition_pre_embedding} yields
\begin{align}
    \norm{g''}_{\lpd{2}{m}}^2
    &\leq
    8C_\ell^4\theta^4_\text{sob} H_f^2
    \norm{u}_{\hild{2}{d}}^4
    +
    2G_f^2
    \norm{u}_{\hild{2}{d}}^2
    \nonumber\\
    &\leq
    2\max\{C_\ell^4\theta^4_\text{sob}, 1\}
    \max\{H_f^2, G_f^2\}
    \left(
        \norm{u}_{\hild{2}{d}}^2
        +
        4\norm{u}_{\hild{2}{d}}^4
    \right)
    \nonumber\\
    &=
    C_{\mathrm{c}}
    \max\{H_f^2, G_f^2\}
    \psi\left(
        \norm{u}_{\hild{2}{d}}^2
    \right),
\end{align}
where $C_{\mathrm{c}} := 2\max\{C^4_\ell\theta^4_\text{sob}, 1\}$. This completes the proof.
\subsection{Proof of Lemma~\ref{lem:kernel_upperbound}}\label{proof_lem:kernel_upperbound}
Using
\(
    |\sin v|\leq1
\)
and
\(
    |\cos v|\leq1,
\)
we have
\begin{align}
    |\Phi(u,v)|
    \leq
    4e^{-u},
    \qquad
    u\geq0.
\end{align}
Applying this bound to \eqref{eq:proof_equivalent_kernel} gives
\begin{align}
    \abs{K_\lambda(x,t)}
    &\leq
    \frac{1}{
        \sqrt{2}\lambda^{1/4}
    }
    e^{-\frac{|x-t|}
    {\sqrt{2}\lambda^{1/4}}}
    +
    \frac{4}{
        \sqrt{2}\lambda^{1/4}
    }
    \left(
        e^{-\frac{x+t}
        {\sqrt{2}\lambda^{1/4}}}
        +
        e^{-\frac{2-x-t}
        {\sqrt{2}\lambda^{1/4}}}
    \right)
    \nonumber\\
    &\leq
    \frac{9}{\sqrt{2}}
    \lambda^{-1/4},
\end{align}
which proves \eqref{eq:kernel_upperbound}.

\subsection{Proof of Lemma~\ref{lem:kernel_weight_distance}}
\label{proof_lem:kernel_weight_distance}

To establish Lemma~\ref{lem:kernel_weight_distance}, we first relate the equivalent kernel $K_\lambda$ to the continuous Green's function $G_\lambda$, and then bound the discrepancy between $G_\lambda$ and its finite-sample counterpart $G_{N,\lambda}$.

\begin{lemma}[{\cite[Theorem~4.1]{messer1993new}}]
\label{lem:kernel_green_distance}
There exists a constant $C_{\mathrm{M}} > 0$ such that for every $\lambda \in (0, 1]$,
\begin{align}\label{eq:kernel_green_distance}
    \sup_{x,t\in[0,1]}
    \abs{
        K_\lambda(x,t) - G_\lambda(x,t)
    }
    \leq
    C_{\mathrm{M}}
    \lambda^{-1/4}
    \exp\left(
        -\frac{1}{\sqrt{2}}
        \lambda^{-1/4}
    \right).
\end{align}
\end{lemma}

Next, we relate the continuous Green's function $G_\lambda$ to the finite-sample smoothing spline weight function $G_{N,\lambda}$.

\begin{lemma}[{\cite[Theorem~2.1]{nychka1995splines}}]
\label{lem:discrete_green_distance}
Under equidistant design points, there exist constants $\mu_0, \mu_1, \mu_2 > 0$ such that whenever $N\lambda^{1/4} > \mu_0$, we have for all $x, t \in [0,1]$,
\begin{align}\label{eq:discrete_green_distance}
    \abs{
        G_{N,\lambda}(x,t) - G_\lambda(x,t)
    }
    &\leq
    \frac{
        \mu_0 \mu_1 N^{-1}\lambda^{-1/4}
    }{
        1 - \mu_0 N^{-1}\lambda^{-1/4}
    }
    \lambda^{-1/4}
    \exp\left(
        -\mu_2
        \abs{x-t}
        \lambda^{-1/4}
    \right).
\end{align}
\end{lemma}

We now combine Lemmas~\ref{lem:kernel_green_distance} and~\ref{lem:discrete_green_distance} to complete the proof. Set $c_0 := 2\mu_0 > 0$. Since $\declamb \in (0, 1]$ and $N\declamb^{1/4} > c_0$, we have $N\declamb^{1/4} > 2\mu_0 > \mu_0$. Therefore, the condition of Lemma~\ref{lem:discrete_green_distance} holds. Moreover, we have
\begin{align}
    \mu_0 N^{-1}\declamb^{-1/4} < \frac{1}{2},
\end{align}
which implies $1 - \mu_0 N^{-1}\declamb^{-1/4} > 1/2$. Consequently, the prefactor in~\eqref{eq:discrete_green_distance} is bounded above by
\begin{align}\label{eq:prefactor_bound}
    \frac{
        \mu_0 \mu_1 N^{-1}\declamb^{-1/4}
    }{
        1 - \mu_0 N^{-1}\declamb^{-1/4}
    }
    <
    \frac{\mu_1 / 2}{1 / 2}
    =
    \mu_1.
\end{align}

By the triangle inequality, for all $x, t \in [0, 1]$,
\begin{align}\label{eq:correct_kernel_triangle}
    \abs{
        G_{N,\declamb}(x,t) - K_{\declamb}(x,t)
    }
    \leq
    \abs{
        G_{N,\declamb}(x,t) - G_{\declamb}(x,t)
    }
    +
    \abs{
        G_{\declamb}(x,t) - K_{\declamb}(x,t)
    }.
\end{align}
Applying Lemmas~\ref{lem:kernel_green_distance} and~\ref{lem:discrete_green_distance} together with~\eqref{eq:prefactor_bound} to~\eqref{eq:correct_kernel_triangle} yields
\begin{align}
    \abs{
        G_{N,\declamb}(x,t) - K_{\declamb}(x,t)
    }
    &\leq
    \frac{
        \mu_0 \mu_1 N^{-1}\declamb^{-1/4}
    }{
        1 - \mu_0 N^{-1}\declamb^{-1/4}
    }
    \declamb^{-1/4}
    \exp\left(
        -\mu_2
        \abs{x-t}
        \declamb^{-1/4}
    \right)
    +
    C_{\mathrm{M}}
    \declamb^{-1/4}
    \exp\left(
        -\frac{1}{\sqrt{2}}
        {\declamb}^{-1/4}
    \right)
    \nonumber\\
    &\leq
    \mu_1 {\declamb}^{-1/4}
    \exp\left(
        -\mu_2
        \abs{x-t}
        {\declamb}^{-1/4}
    \right)
    +
    C_{\mathrm{M}}
    {\declamb}^{-1/4}
    \exp\left(
        -\frac{1}{\sqrt{2}}
        {\declamb}^{-1/4}
    \right).
\end{align}
Defining $C_W := \mu_1$ and $\kappa := \mu_2$ gives~\eqref{eq:kernel_weight_distance} and completes the proof.

\end{document}